%% file: bmvc_final.tex
\documentclass{bmvc2k}
\usepackage{placeins}
\usepackage[table]{xcolor}
\usepackage{multirow}
\definecolor{oursgreen}{RGB}{226,245,226}
\usepackage{xcolor}
\usepackage{amssymb}

\usepackage{soul}

\usepackage{placeins}
\usepackage[table]{xcolor}
\usepackage{multirow}
\usepackage{amssymb}
\usepackage{amsmath}
\usepackage{amsthm}

\definecolor{momo}{RGB}{255,210,225}

\title{Stochastic Nonlinearities Improve Uncertainty Estimation}

\addauthor{Mohammad Sarhangzadeh}{m.sarhangzadeh@bilkent.edu.tr}{1}
\addauthor{Amirreza M. Shebly}{amirreza@bilkent.edu.tr}{1}
\addauthor{Doruk Oner}{doruk.oner@bilkent.edu.tr}{1}

\addinstitution{
NeuraVision Lab, Bilkent University\\
Ankara, Turkey
}

\runninghead{Sarhangzadeh \etal}{Stochastic Activations for Uncertainty}

\def\etal{\emph{et al}\bmvaOneDot}

\begin{document}

\maketitle
\input{Sections/abstract}
\input{Sections/Intro}

\input{Sections/RelatedWork}
\input{Sections/Method}
\input{Sections/Experiments}
\input{Sections/Ablation}
\input{Sections/Conclusion}

\FloatBarrier
\bibliography{egbib}

\clearpage
\input{Sections/supp}

\end{document}

%% file: Sections/abstract.tex
\begin{abstract}
Reliable uncertainty estimation is important across visual recognition tasks, providing information about the confidence and reliability of model predictions.
Many uncertainty estimation methods obtain predictive diversity through mechanisms that can suppress intermediate features, perturb the input, or require substantial additional computation. 
Stochastic activation functions, where network units randomly select between multiple nonlinearities during a forward pass, have previously been used to improve optimization behavior, mitigate limitations of fixed ReLU activations, and increase output diversity. 
In this work, we reframe stochastic activation switching as a lightweight mechanism for uncertainty estimation. 
The proposed method samples between nonlinear activation functions, inducing predictive diversity while keeping intermediate feature pathways active. 
This enables uncertainty estimation from a single trained model without the training cost of ensembles, task-specific uncertainty heads, or architectural redesign.
We evaluate the proposed approach across classification, semantic segmentation, and out-of-distribution detection using convolutional and transformer-based architectures on diverse datasets.
Across these settings, stochastic activation switching improves uncertainty quality while preserving competitive primary-task performance, showing that stochastic nonlinearities provide an effective and information-preserving source of predictive variation for visual recognition.
\end{abstract}

%% file: Sections/Intro.tex
\section{Introduction}
Deep neural networks (DNNs) \cite{goodfellow2016deep} have driven major advances in computer vision, achieving state-of-the-art performance across image classification \cite{krizhevsky2012imagenet}, object detection \cite{redmon2016lookonceunifiedrealtime, ren2016fasterrcnnrealtimeobject}, semantic segmentation, and dense prediction tasks \cite{ronneberger2015unetconvolutionalnetworksbiomedical, minaee2021image}. 
Despite this progress, vision models frequently produce overconfident and poorly calibrated predictions \cite{guo2017calibration}, even when confronted with ambiguous inputs, corrupted samples, or out-of-distribution data \cite{hendrycks2016gaussian}. 
This inability to reliably ``know what it does not know'' remains an obstacle to deploying deep learning systems in real-world visual recognition pipelines, where overconfident errors can lead to unreliable downstream decisions.
Uncertainty estimation addresses this limitation by augmenting predictions with confidence measures that help identify unreliable outputs \cite{kendall2017uncertainties}. 
However, obtaining uncertainty estimates that are accurate, well-calibrated, computationally efficient, and compatible with modern architectures remains a challenging problem~\cite{abdar2021review}.

A common strategy for uncertainty estimation is to generate multiple predictions for the same input and measure their disagreement. 
Deep ensembles~\cite{lakshminarayanan2017simple} achieve this by training multiple independent models and are often considered a strong baseline for predictive uncertainty. 
However, training and storing multiple models is computationally expensive.
Single-model alternatives such as Monte Carlo dropout~\cite{gal2016dropout} reduce this cost by injecting stochasticity into the network during inference. 
Nevertheless, dropout-based methods obtain diversity by randomly suppressing features, which can disrupt information flow and degrade the intermediate representations required for accurate classification or dense prediction. 
This trade-off motivates the need for uncertainty estimation mechanisms that can induce predictive diversity without disrupting learned feature representations.

\begin{figure}
\centering
\includegraphics[width=1.00\linewidth]{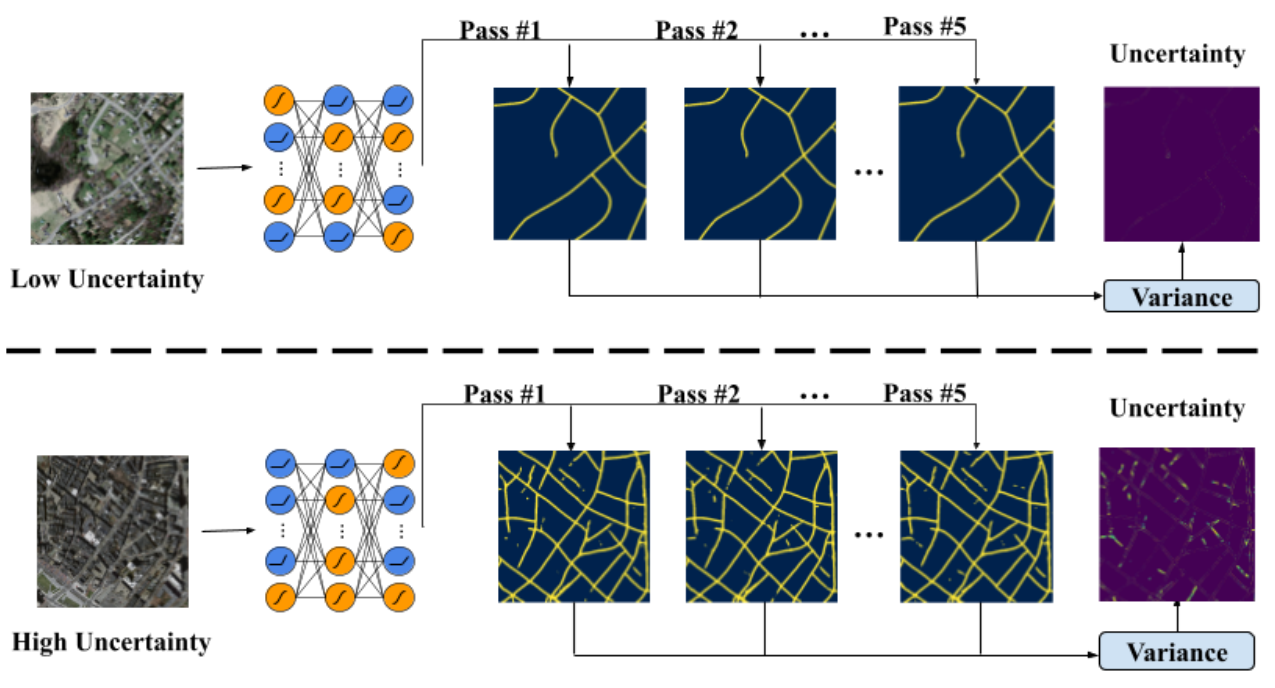}
\caption{
\textbf{Uncertainty from stochastic nonlinearities.}
Stochastic forward passes are shown for a network trained to predict road maps.
Clear roads produce similar outputs and low variance, while ambiguous roads and road-like buildings produce higher variance.
}
\label{fig:uncertainty_visualization}
\end{figure}

Stochastic activation functions have previously been explored as a means to improve gradient flow, introduce regularization, and increase predictive diversity in deep networks. 
Instead of using a fixed nonlinearity for each neuron or layer, stochastic activation methods allow activation units to randomly select from a predefined set of nonlinear functions during each forward pass. 
Prior work has used this randomness to improve robustness, alleviate the dead-unit problem, and encourage more diverse network behavior~\cite{Electrica_2021_486148,gulcehre2016noisy,xu2015empirical}. 
However, these methods have primarily been studied from the perspective of optimization and regularization, rather than a mechanism for uncertainty estimation.

In this work, we reframe stochastic activation functions as a mechanism for uncertainty quantification.
Our main observation is that stochastic switching of nonlinearities can introduce useful predictive variation while preserving the flow of information through the network. 
Rather than removing features, the method keeps intermediate representations available and perturbs the nonlinear transformation applied to them. 
As a result, a single trained model can produce multiple plausible outputs without discarding intermediate information (Fig.~\ref{fig:uncertainty_visualization}).
This property is particularly relevant for image classification and semantic segmentation, where preserving semantic and spatial structure is important for reliable predictions.

We propose a simple framework for uncertainty estimation based on stochastic activation switching. 
The method does not require training multiple models or substantially redesigning the underlying architecture. 
In our experiments, we apply it to both convolutional networks and transformer-based \cite{vaswani2017attention} vision models to evaluate its applicability across different model classes. 
At inference time, repeated stochastic forward passes produce a distribution of predictions from a single set of learned weights, enabling uncertainty to be estimated from predictive disagreement. 
Because the stochasticity is introduced through activation selection rather than feature removal, the proposed approach provides an efficient alternative to feature-masking methods while avoiding the computational burden of explicit ensembles.

We evaluate the framework on classification and segmentation across diverse datasets, modalities, scales, and in- and out-of-distribution settings. We compare against common uncertainty baselines and analyze task performance and uncertainty quality, with ablations on stochastic activation and inference settings. Results show competitive predictive performance and strong uncertainty estimates from a single trained model.

The main contributions of this work are as follows: (i) We propose stochastic nonlinearities as a computationally efficient and information-preserving approach for uncertainty estimation in image classification and semantic segmentation. (ii) We show that the proposed single-model approach achieves competitive uncertainty calibration and predictive performance compared with commonly used uncertainty estimation baselines, while avoiding the cost of training multiple models. (iii) We conduct extensive experiments and ablation across small, medium, and large-scale datasets, different visual modalities, convolutional and transformer-family architectures, and both in-distribution and out-of-distribution settings.

%% file: Sections/RelatedWork.tex
\section{Related Work}
\subsection{Uncertainty Estimation}

Bayesian Neural Networks (BNNs) \cite{blundell2015weight} model epistemic uncertainty through posterior distributions over network parameters, but approximate Bayesian inference is costly at scale. 
Laplace approximations \cite{daxberger2021laplace} reduce this cost by approximating the posterior around a trained solution, though performance depends on the quality of the approximation and curvature estimate. 
Lightweight probabilistic networks \cite{gast2018lightweight} instead propagate activation uncertainty using assumed density filtering, offering a more efficient alternative but relying on approximate distributional assumptions.

Monte Carlo Dropout (MC Dropout) \cite{gal2016dropout} estimates uncertainty by keeping dropout active during inference and aggregating stochastic predictions. 
Although simple, lightweight, and model-agnostic \cite{loquercio2020general}, it obtains diversity by randomly suppressing features, so each stochastic pass discards part of the learned representation rather than transforming the same information differently. Variants such as MC-DropConnect \cite{mobiny2021dropconnect}, which stochastically masks network weights, Concrete Dropout \cite{gal2017concrete}, which learns dropout rates, and R-drop \cite{wu2021r}, which uses auxiliary regularization, have also been proposed, but their uncertainty still comes from removing information rather than preserving and transforming it.

Deep Ensembles \cite{lakshminarayanan2017simple} estimate uncertainty from disagreement among multiple independently trained networks. 
While effective, their diversity comes at the cost of repeated training, storage, and inference, with the computational burden typically scaling with the number of ensemble members. 
Efficiency-driven alternatives reduce this overhead by using fixed partially overlapping masks as in Masksembles \cite{durasov2021masksembles}, training checkpoints \cite{huang2017snapshot}, low-rank parameter sharing \cite{wen2020batchensemble}, or multi-branch architectures \cite{teerapittayanon2016branchynet}, but fully independent ensembles often remain stronger in calibration and reliability \cite{ovadia2019can}.

Deterministic single-forward-pass methods provide an alternative that avoids repeated stochastic inference, including DUQ \cite{van2020uncertainty} and SNGP \cite{liu2020simple}.
More recently, Deep Deterministic Uncertainty (DDU) \cite{mukhoti2023deep} combines a regularized feature representation with post-hoc density estimation to obtain uncertainty from a single deterministic forward pass.
Although these methods offer low inference cost, their uncertainty depends on the learned feature geometry and the assumptions of the distance based estimator. 

Other single-model methods estimate uncertainty through input or model perturbations at inference time. 
Test-time augmentation (TTA) \cite{wang2019aleatoric} uses transformed inputs, but requires manually chosen augmentations and repeated inference. 
Monte Carlo Batch Normalization \cite{mukhoti2020batch} samples stochastic batch-normalization statistics, while iterative methods \cite{durasov2024enabling} use repeated refinement or convergence behavior. 
HCM \cite{choi2026uncertainty} provides a sampling-free alternative based on hyperspherical confidence, but depends on its geometric formulation and optimization. 
These methods reduce the need for independently trained ensembles, but rely on input transformations, normalization noise, iterative refinement, or geometric constraints.

\subsection{Stochastic activations}
Stochastic activations have traditionally been used for regularization, optimization, and, in recent work, generative diversity. Early forms, such as dropout \cite{hinton2012improvingneuralnetworkspreventing}, can be interpreted as stochastic activation mechanisms that prevent co-adaptation of features. For optimization, methods such as those proposed by \cite{gulcehre2016noisy} address specific pathologies in ReLU networks, including the ``dying ReLU'' phenomenon \cite{xu2015empirical}, by injecting controlled randomness to preserve gradient flow in saturated regions and improve convergence stability.

Recent studies, including \cite{lomeli2025stochastic} and others exploring noise injection and adaptive activations \cite{duan2024optimized, hammad2024deep}, explored stochastic activations primarily to enhance training dynamics, sparsity, and runtime adaptability. These methods switch between activation functions to improve optimization efficiency and occasionally preserve stochasticity during inference for improved generative behavior. In language and generative modeling, Lomeli et al.~\cite{lomeli2025stochastic} demonstrated that such test-time stochasticity can serve as an alternative to temperature sampling \cite{holtzman2019curious} for increasing output diversity.

A related activation-based approach, Random Activation Functions (RAF) \cite{stoyanova2023toward}, uses activation-function diversity for uncertainty estimation through an ensemble of independently trained models, each assigned a randomly selected activation function that remains fixed within that model. Predictive uncertainty is estimated from disagreement across the ensemble members. In contrast, our method dynamically samples activation choices within a single model during both training and inference. Our work therefore leverages stochastic activations explicitly for uncertainty estimation. By maintaining activation stochasticity during inference, a single model can generate a functionally diverse set of predictions, providing a measure of combined predictive uncertainty without feature masking or architectural duplication. Similar to MC-Dropout, this uncertainty is derived from disagreement across repeated stochastic forward passes and is not explicitly decomposed into separate epistemic and aleatoric components. This reframes stochastic activations not only as a training mechanism, but also as an efficient, information-preserving path to reliable uncertainty estimation.

%% file: Sections/Method.tex
\section{Method}
\label{sec:method}

\subsection{Motivation}
As discussed in the previous sections, existing uncertainty estimation methods often obtain predictive diversity by either suppressing features, training multiple models, perturbing the input, or introducing additional inference procedures. 
These choices can make uncertainty estimation less information-preserving, computationally expensive, or dependent on design choices that must be tuned for each architecture and dataset. 
Our goal is to obtain diverse predictions from a single trained model without removing intermediate features, duplicating the network, or introducing additional trainable parameters. 

We achieve this by replacing deterministic activation functions with stochastic nonlinearities. 
Instead of assigning a fixed activation function to each unit, the network randomly switches between the backbone's original activation and an alternative nonlinearity during each forward pass. 
The model is trained with the same stochastic mechanism used at inference time, so the learned weights become robust to activation-level variation rather than treating it as a test-time perturbation only.
At test time, the weights are kept fixed and multiple stochastic forward passes are performed through the same model. 
The resulting prediction variability is then used to estimate uncertainty.

This formulation introduces diversity through transformations rather than feature removal. Every activation pathway remains active, and stochasticity only changes how the intermediate signal is transformed. As illustrated in Fig.~\ref{fig:uncertainty_quality}, this leads to more structured uncertainty than feature masking when stochasticity is increased at inference time. Therefore, the network can sample different functional behaviors while preserving the information carried by its features.
The method does not require model-specific uncertainty heads, additional learnable parameters, or repeated model training, making it simple to apply across different classification and segmentation architectures.

\begin{figure}[h]
\centering
\includegraphics[width=0.9\columnwidth]{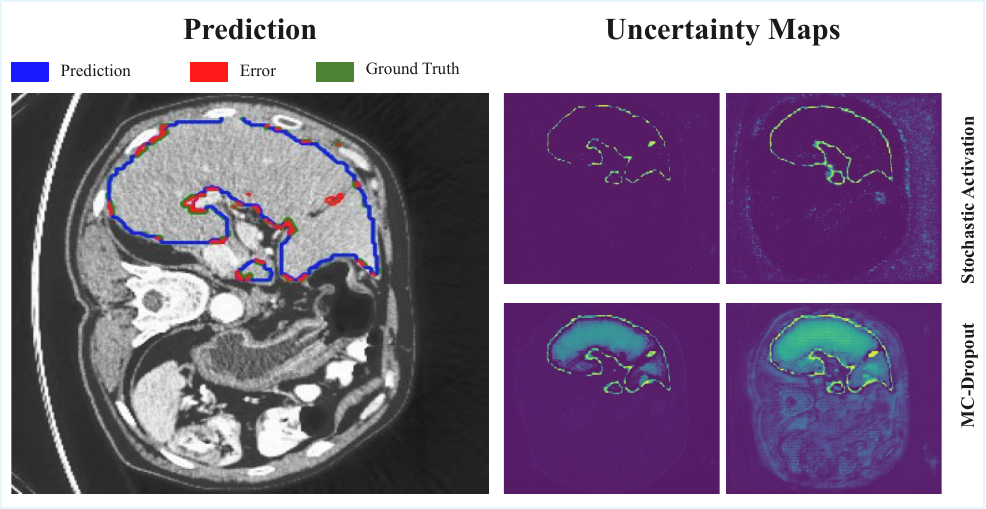}
\caption{
\textbf{Uncertainty quality under test-time stress.}
The left panel shows the input, prediction, ground truth, and errors; the right panels compare uncertainty maps from stochastic activation switching and MC Dropout, with brighter values indicating higher uncertainty.
We increase stochasticity only at inference time to examine its effect on uncertainty.
MC Dropout removes intermediate features, causing uncertainty to spread broadly, whereas stochastic activation switching keeps feature pathways active and concentrates uncertainty near boundaries and prediction errors.
This supports our motivation that activation-level stochasticity can induce predictive diversity without disrupting intermediate representations.
}
\label{fig:uncertainty_quality}
\end{figure}

\subsection{Stochastic Activation Switching}
\label{sec:stochastic_activation_switching}

Let $f_{\theta}$ denote a neural network with parameters $\theta$. 
In a standard deterministic network, each pre-activation $z_{j}^{(l)}$ at unit $j$ in layer $l$ is passed through a fixed activation function chosen by the backbone architecture. 
We replace this fixed choice with a stochastic activation switch. 
During each stochastic forward pass $t$, the unit applies either the original activation of the backbone, denoted by $\sigma_{\mathrm{orig}}^{(l)}$, or an alternative nonlinearity $\sigma_{\mathrm{alt}}$:
\begin{equation}
h_{j}^{(l,t)} =
\begin{cases}
\sigma_{\mathrm{orig}}^{(l)}\!\left(z_{j}^{(l,t)}\right), & \text{with probability } 1-p,\\[4pt]
\sigma_{\mathrm{alt}}\!\left(z_{j}^{(l,t)}\right), & \text{with probability } p.
\end{cases}
\label{eq:activation_switching}
\end{equation}
Equivalently, let $m_{j}^{(l,t)} \sim \mathrm{Bernoulli}(p)$ denote the binary activation-switching variable. 
Then Eq.~\ref{eq:activation_switching} can be written as
\begin{equation}
h_{j}^{(l,t)}
=
\left(1-m_{j}^{(l,t)}\right)
\sigma_{\mathrm{orig}}^{(l)}\!\left(z_{j}^{(l,t)}\right)
+
m_{j}^{(l,t)}
\sigma_{\mathrm{alt}}\!\left(z_{j}^{(l,t)}\right).
\label{eq:bernoulli_activation_switching}
\end{equation}
Here, $p$ controls the probability of switching away from the backbone activation. 
In our experiments, $\sigma_{\mathrm{orig}}^{(l)}$ is the native activation used by the architecture, such as ReLU in convolutional networks or GELU in transformer-family models, while $\sigma_{\mathrm{alt}}$ is chosen as an alternative differentiable nonlinearity.
Unless otherwise stated, we use $p=0.5$, corresponding to uniform switching between the two nonlinearities. This choice maximizes the entropy of the Bernoulli activation selection, giving both activation pathways equal probability and avoiding a systematic preference for either nonlinearity.

The sampling operation can be applied at different levels of granularity, leading to different forms of functional diversity:
\begin{itemize}
    \item \textbf{Node-level:} The most granular approach, where a unique activation is sampled for each individual activation unit $j$ in layer $l$.
    
    \item \textbf{Channel-level:} A strategy primarily for convolutional networks, where a single activation is sampled for each channel $c$ in layer $l$ and shared across all spatial locations within that channel.
    
    \item \textbf{Layer-level:} The coarsest strategy, where one activation is sampled for the entire layer $l$ and applied to all units in that layer.
\end{itemize}

Regardless of the sampling granularity, each stochastic forward pass defines a different activation pathway through the same trained weights. 
For an input $x$, the $t$-th pass produces:
\begin{equation}
\bar{y}^{(t)} = f_{\theta}\!\left(x;\sigma^{(t)}\right)
\label{eq:stochastic_forward_pass}
\end{equation}
where $\sigma^{(t)}$ denotes the sampled activation configuration used during that pass. 
Over $T$ stochastic passes, the model generates a sequence of predictions
\begin{equation}
Y(x)=\left\{\bar{y}^{(1)},\bar{y}^{(2)},\dots,\bar{y}^{(T)}\right\}
\label{eq:stochastic_activation_trajectory}
\end{equation}
all produced by the same parameters $\theta$ but different activation pathways. 
We refer to this sequence as a \textit{stochastic activation trajectory}.

Figure~\ref{fig:toy_regression} provides an intuitive example of this behavior in a toy regression setting with heteroscedastic noise.
Compared with MC-Dropout, our stochastic activation trajectories remain more stable where the data are well supported, while their variation increases in noisier regions and outside the training range.

\begin{figure}[t]
\centering
\includegraphics[width=0.9\columnwidth]{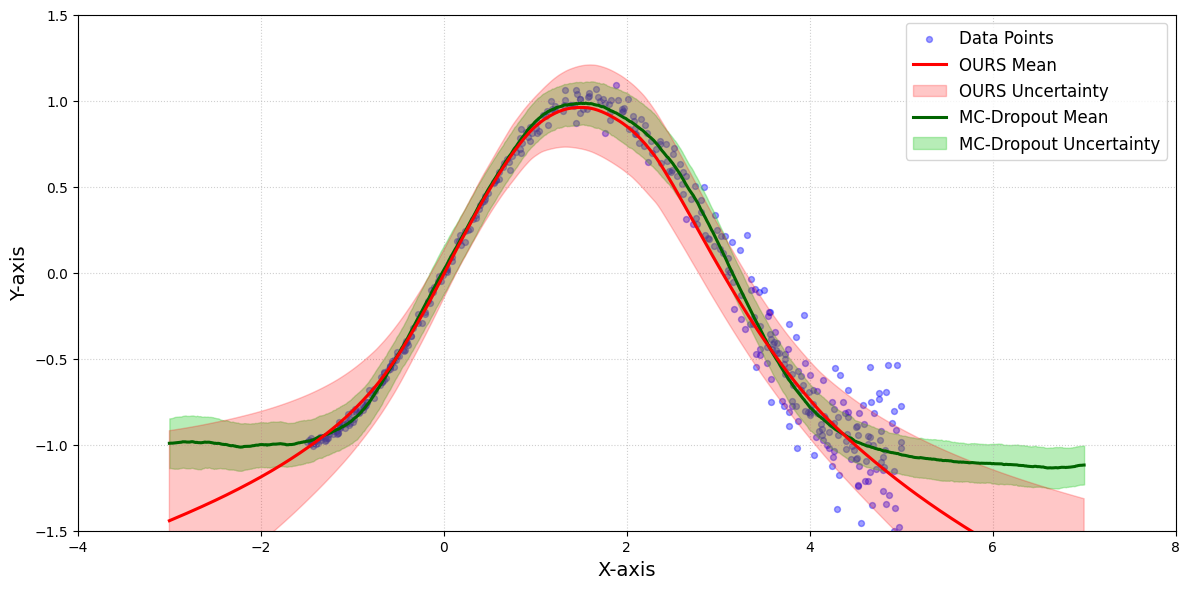}
\caption{
\textbf{Toy regression under heteroscedastic noise.}
A 3-layer MLP is trained on noisy sinusoidal samples, and stochastic activation switching is compared with MC-Dropout under the same setup.
Both methods use repeated stochastic passes.
Our method shows clearer combined predictive uncertainty, remaining more certain in well-supported regions while uncertainty increases with stronger noise and outside the observed range.
}
\label{fig:toy_regression}
\end{figure}

\subsection{Training and Uncertainty Estimation}
\label{sec:training_uncertainty}

Training follows the standard supervised learning procedure for the target task. 
Given an input--label pair $(x,y)$, the model performs a stochastic forward pass using activation choices sampled according to Eq.~\ref{eq:activation_switching}. 
The network parameters $\theta$ are optimized using the task loss:
\begin{equation}
\min_{\theta} \; \mathcal{L}\!\left(f_{\theta}(x; \sigma^{(t)}), y\right),
\label{eq:training_objective}
\end{equation}
where $\sigma^{(t)}$ denotes the sampled activation configuration for forward pass $t$. 
No additional uncertainty objective, auxiliary head, or architecture-specific parameter is introduced. 
The same stochastic activation mechanism is used during training and inference, but the weights are updated only during training.

At inference time, the learned weights $\theta$ are fixed. 
For an input $x$, we perform $T$ stochastic forward passes by repeatedly sampling activation configurations.
Each pass produces a probability prediction $p^{(t)}(x) = f_{\theta}\!\left(x; \sigma^{(t)}\right)$,
where $p^{(t)}(x)$ denotes class probabilities for classification or pixel-wise class probabilities for segmentation. 
At inference time, the final prediction and uncertainty are computed as
\begin{equation}
\hat{\mu}(x)=\frac{1}{T}\sum_{t=1}^{T} p^{(t)}(x),
\qquad
\hat{U}(x)=\frac{1}{T}\sum_{t=1}^{T}
\left(p^{(t)}(x)-\hat{\mu}(x)\right)^2.
\label{eq:predictive_mean_uncertainty}
\end{equation}
Here, $\hat{\mu}(x)$ is used as the final prediction, and $\hat{U}(x)$ denotes the empirical variance across stochastic predictions. 
For classification, this variance is computed over class probabilities. 
For segmentation, it is computed at each pixel, producing a spatial uncertainty map. 
High uncertainty indicates that different stochastic activation configurations lead to inconsistent predictions, while low uncertainty indicates stable predictions across activation choices.

To analyze whether the stochastic mechanism preserves information more effectively than feature-masking approaches, we introduce \textbf{Information Destruction Rate (IDR)} as a diagnostic measure. 
IDR is not used as a training objective; it is used only to compare the effect of inference-time stochasticity on feature representations. 
Let $I_{\mathrm{full}}^{(l)} = I(h_{\mathrm{in}}^{(l)}; h_{\mathrm{out}}^{(l)})$ denote the estimated mutual information \cite{zeevi2025rate} between the input and output of layer $l$ under the deterministic forward pass. 
For stochastic pass $t$, let $I_t^{(l)}$ be the corresponding mutual-information estimate under stochastic inference. 
We define
\begin{equation}
\mathrm{IDR}^{(l)}
=
1 - \frac{1}{T}\sum_{t=1}^{T}
\frac{I_t^{(l)}}{I_{\mathrm{full}}^{(l)}}.
\label{eq:idr}
\end{equation}
Lower IDR indicates that stochastic inference preserves more of the deterministic layer-wise input--output information. 
We use this analysis to compare stochastic activation switching with feature-masking baselines such as MC Dropout \cite{gal2016dropout}, where uncertainty is induced by removing activations rather than transforming them.

%% file: Sections/Experiments.tex
\section{Experiments}
\label{sec:experiments}

\subsection{Experimental Setup}
\label{sec:experimental_setup}

We evaluate the proposed stochastic activation switching framework across image segmentation and image classification tasks. 
The experiments are designed to test whether the method provides reliable uncertainty estimates while maintaining competitive task performance across different domains, architectures, and distribution shifts. 
For fair comparison, all methods are trained and evaluated using the same dataset splits, input resolution, batch size, preprocessing, augmentation pipeline, backbone architecture, and optimization schedule within each task, unless a baseline requires its standard uncertainty-specific modification.

\textbf{Baselines.}
We compare against MC Dropout \cite{gal2016dropout}, DDU \cite{mukhoti2023deep}, TTA \cite{wang2019aleatoric}, Laplace Redux \cite{daxberger2021laplace}, Masksembles \cite{durasov2021masksembles}, Deep Ensembles \cite{lakshminarayanan2017simple}, and, for segmentation, an Iterative Network baseline. 
MC Dropout, TTA, Masksembles, and our method are single-model multi-pass methods, while Laplace Redux applies a post-hoc posterior approximation to a trained model and DDU applies post-hoc density estimation to learned features. 
Deep Ensembles require five independently trained models, Masksembles uses five fixed masks, and the Iterative baseline uses three refinement iterations~\cite{durasov2024enabling}. 
Unless otherwise stated, our method and MC Dropout use five stochastic forward passes, TTA uses five augmented inference samples.
For Laplace Redux, we follow the sampling configuration appropriate to each dataset and backbone. 
We report relative training and inference cost where appropriate.

\textbf{Datasets.}
For segmentation, we use Massachusetts Roads \cite{MnihThesis}, DRIVE \cite{staal2004ridge}, CREMI \cite{funke2018large}, and LiTS \cite{bilic2023liver}, with RoadTracer \cite{bastani2018roadtracerautomaticextractionroad}, CHASE-DB1 \cite{fraz2012ensemble}, and ISBI 2012 \cite{cardona2010drosophila} used for cross-dataset OOD evaluation. 
For classification, we use CIFAR-10 \cite{krizhevsky2009learning}, MiniImageNet \cite{vinyals2016matching}, BloodMNIST, and PathMNIST \cite{yang2021medmnist, yang2023medmnist}, with SVHN \cite{netzer2011reading} and CIFAR-100 \cite{krizhevsky2009learning} used as OOD datasets for CIFAR-10 and MiniImageNet, respectively. , and BloodMNIST and PathMNIST used as cross-domain OOD sets for each other.

\textbf{Backbones.}
For segmentation, we use U-Net-based architectures \cite{ronneberger2015unetconvolutionalnetworksbiomedical} for Massachusetts Roads, DRIVE, and LiTS, and SwinTransformer V2 \cite{liu2022swin, conde2022swin2sr} for CREMI. 
For classification, we use ResNet-style convolutional backbones for CIFAR-10, and MiniImageNet. 
For high-resolution MedMNIST classification, we use SwinTransformer V2 \cite{liu2022swin, conde2022swin2sr}. 
In all cases, our method keeps the same backbone as the corresponding baseline and only replaces deterministic activations with stochastic activation switching.
Unless otherwise stated, stochasticity is applied at the node level, as this provides the best overall trade-off in our granularity ablation, as shown in Section~\ref{sec:ablation_analysis}. 
For segmentation experiments, we switch between the backbone's native activation and Tanh, while for classification experiments we use the backbone's native activation and SiLU \cite{elfwing2018sigmoid}, following the activation-choice ablation.

\textbf{Metrics.}
For segmentation, we report F1 and structural overlap metrics: Correctness, Completeness, and Quality (CCQ) \cite{wiedemann1998empirical}, together with Pearson correlation between uncertainty and prediction error, expected calibration error (ECE) \cite{naeini2015obtaining}, and Brier score \cite{glenn1950verification}. 
Segmentation OOD detection uses ROC-AUC and PR-AUC. 
For classification, we report accuracy, ECE, Brier score, and OOD metrics AUC, AUPR, and FPR95 \cite{malinin2018predictive}. 
Classification directly yields one image-level uncertainty score, while for segmentation we spatially average the pixel-wise uncertainty map to obtain one score per image. OOD metrics are computed from these image-level scores. For segmentation and CIFAR-10 classification, we report mean $\pm$ standard deviation over three random seeds.

\begin{figure*}[t]
\center
\includegraphics[width=0.9\textwidth]{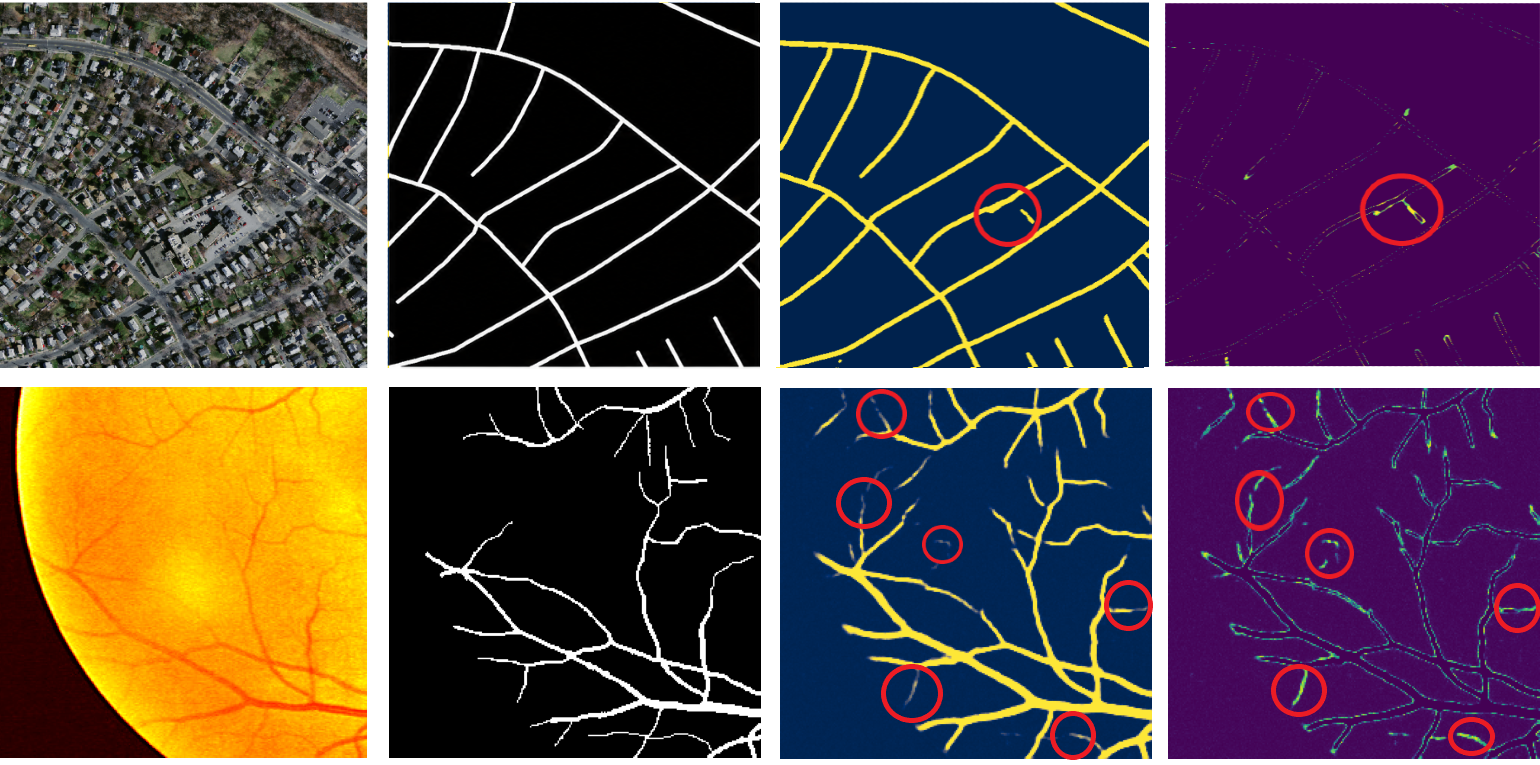}
\caption{
Visualization of our stochastic activation based uncertainty estimation.
From left to right: input image, ground‐truth mask, model prediction, and
the uncertainty map computed from the variance across $5$ stochastic
activation trajectories.
}
\label{fig:uncertainty_example}
\end{figure*}

\begin{table}[!t]
\begin{center}
\resizebox{\columnwidth}{!}{%
\begin{tabular}{|l|l|c|c|c|c|}
\hline
\textbf{Dataset} & \textbf{Method} & \textbf{Correct. (↑)} & \textbf{Complete. (↑)} & \textbf{Qual. (↑)} & \textbf{F1 (↑)} \\
\hline\hline

\multirow{9}{*}{Massachusetts}
& MC Dropout     & 91.00 & 87.20 & 81.90 & 78.62 \\
& DDU            & 91.44 & 86.10 & 80.52 & 79.23 \\
& TTA            & 91.46 & 86.73 & 81.19 & 79.43 \\
& Masksembles    & 81.63 & 87.67 & 81.94 & \underline{80.02} \\
& Deep Ensemble  & \textbf{92.37} & \underline{88.11} & \textbf{82.67} & \textbf{80.54} \\
& Laplace Redux  & 91.45 & 86.11 & 80.52 & 79.23 \\
& Iterative      & 90.07 & 84.97 & 78.39 & 78.58 \\
& \cellcolor{oursgreen}\textbf{Ours}
& \cellcolor{oursgreen}\underline{91.88}
& \cellcolor{oursgreen}\textbf{88.24}
& \cellcolor{oursgreen}\underline{82.11}
& \cellcolor{oursgreen}79.71 \\
& \cellcolor{oursgreen}
& \cellcolor{oursgreen}$\pm$0.43
& \cellcolor{oursgreen}$\pm$0.38
& \cellcolor{oursgreen}$\pm$0.65
& \cellcolor{oursgreen}$\pm$0.16 \\
\hline

\multirow{9}{*}{DRIVE}
& MC Dropout     & 96.18 & 85.13 & 82.53 & 79.41 \\
& DDU           & 96.13 & 85.83 & 82.86  & 80.81 \\
& TTA            & 95.45 & 84.41 & 82.29 & 79.98 \\
& Masksembles    & 96.82 & 85.11 & 82.69 & \underline{81.02} \\
& Deep Ensemble  & \underline{97.28} & \textbf{85.84} & \textbf{83.72} & \textbf{81.84} \\
& Laplace Redux  & 96.21 & 84.54 & 82.81 & 77.56 \\
& Iterative      & 96.86 & 82.61 & 81.87 & 80.30 \\
& \cellcolor{oursgreen}\textbf{Ours}
& \cellcolor{oursgreen}\textbf{98.38}
& \cellcolor{oursgreen}\underline{85.61}
& \cellcolor{oursgreen}\underline{82.91}
& \cellcolor{oursgreen}81.00 \\
& \cellcolor{oursgreen}
& \cellcolor{oursgreen}$\pm$0.84
& \cellcolor{oursgreen}$\pm$0.48
& \cellcolor{oursgreen}$\pm$0.53
& \cellcolor{oursgreen}$\pm$0.28 \\
\hline

\multirow{9}{*}{LiTS}
& MC Dropout     & 87.09 & 71.06 & 65.82 & 96.41 \\
& DDU            & 87.45 & 78.82 & 69.73 & 97.02 \\
& TTA            & 72.04 & 64.01 & 58.73 & 93.87 \\
& Masksembles    & 88.27 & 77.83 & 70.75 & 98.01 \\
& Deep Ensemble  & 87.59 & \textbf{81.38} & \textbf{72.84} & \textbf{98.25} \\
& Laplace Redux  & 87.27 & \underline{80.40} & 71.43 & 97.51 \\
& Iterative      & \underline{87.98} & 74.55 & 70.26 & 97.39 \\
& \cellcolor{oursgreen}\textbf{Ours}
& \cellcolor{oursgreen}\textbf{89.39}
& \cellcolor{oursgreen}79.57
& \cellcolor{oursgreen}\underline{72.93}
& \cellcolor{oursgreen}\underline{98.11} \\
& \cellcolor{oursgreen}
& \cellcolor{oursgreen}$\pm$0.44
& \cellcolor{oursgreen}$\pm$0.35
& \cellcolor{oursgreen}$\pm$0.59
& \cellcolor{oursgreen}$\pm$0.10 \\
\hline

\multirow{9}{*}{CREMI}
& MC Dropout     & \underline{97.40} & 94.45 & 92.68 & 78.58 \\
& DDU            & 96.41 & 94.76 & 93.16 & 79.48 \\
& TTA            & 97.04 & 94.65 & 92.63 & 76.23 \\
& Masksembles    & 96.90 & 95.01 & 93.56 & \underline{81.76} \\
& Deep Ensemble  & 97.24 & \underline{95.83} & \underline{93.93} & 79.92 \\
& Laplace Redux  & 97.26 & 95.62 & 93.76 & 79.58 \\
& Iterative      & 96.02 & 87.90 & 87.19 & 65.73 \\
& \cellcolor{oursgreen}\textbf{Ours}
& \cellcolor{oursgreen}\textbf{98.47}
& \cellcolor{oursgreen}\textbf{96.61}
& \cellcolor{oursgreen}\textbf{94.76}
& \cellcolor{oursgreen}\textbf{81.94} \\
& \cellcolor{oursgreen}
& \cellcolor{oursgreen}$\pm$0.55
& \cellcolor{oursgreen}$\pm$0.43
& \cellcolor{oursgreen}$\pm$0.71
& \cellcolor{oursgreen}$\pm$0.11 \\
\hline

\end{tabular}%
}
\end{center}
\caption{Segmentation performance across aerial and medical benchmarks.}
\label{tab:seg_accuracy}
\end{table}

\subsection{Results}
\label{sec:segmentation_experiments}

\textbf{Segmentation} Tables~\ref{tab:seg_accuracy} and~\ref{tab:seg_uncertainty} show that stochastic activation switching achieves best or second-best results in most task and uncertainty metrics while retaining the original architecture and single-model training setup. Table~\ref{tab:seg_accuracy} shows strong segmentation performance on both thin structures such as roads and vessels and more complex biomedical settings.

This indicates that activation-level stochasticity can support uncertainty estimation while retaining strong spatial representations for dense prediction. Table~\ref{tab:seg_uncertainty} shows that the resulting uncertainty aligns well with prediction errors, with strong calibration and Brier scores.
Together with Table~\ref{tab:seg_accuracy}, these results show improved uncertainty quality without sacrificing segmentation performance or introducing task-specific uncertainty modules.

For OOD evaluation, models trained on Massachusetts, DRIVE, and CREMI are tested on RoadTracer, CHASE-DB1, and ISBI 2012, respectively.
Table~\ref{tab:seg_ood} shows that disagreement between stochastic activation trajectories remains informative under distribution shift.

\begin{table}[!t]
\begin{center}
\resizebox{\columnwidth}{!}{%
\begin{tabular}{|l|l|c|c|c|c|c|}
\hline
\textbf{Dataset} & \textbf{Method} & \textbf{Corr. (↑)} & \textbf{ECE (↓)} & \textbf{Brier (↓)} & \textbf{Train} & \textbf{Infer} \\
\hline\hline

\multirow{9}{*}{Massachusetts}
& MC Dropout     & 79.73 & 0.026 & 0.033 & $\times1$ & $\times5$ \\
& DDU            & 75.64 & 0.027 & 0.031 & $\times1$ & -- \\
& TTA            & 79.90 & 0.025 & 0.029 & $\times1$ & $\times5$ \\
& Masksembles    & 79.62 & \underline{0.024} & 0.031 & $\times1$ & $\times5$ \\
& Deep Ensemble  & \underline{80.03} & \underline{0.024} & \underline{0.027} & $\times5$ & $\times5$ \\
& Laplace Redux  & 77.51 & 0.034 & 0.034 & $\times1$ & -- \\
& Iterative      & 77.28 & 0.028 & 0.030 & $\times3$ & $\times3$ \\
& \cellcolor{oursgreen}\textbf{Ours}
& \cellcolor{oursgreen}\textbf{80.38}
& \cellcolor{oursgreen}\textbf{0.020}
& \cellcolor{oursgreen}\textbf{0.025}
& \cellcolor{oursgreen}$\times1$
& \cellcolor{oursgreen}$\times5$ \\
& \cellcolor{oursgreen}
& \cellcolor{oursgreen}$\pm$1.08
& \cellcolor{oursgreen}$\pm$0.002
& \cellcolor{oursgreen}$\pm$0.006
& \cellcolor{oursgreen}
& \cellcolor{oursgreen} \\
\hline

\multirow{9}{*}{DRIVE}
& MC Dropout     & 80.58 & 0.032 & 0.040 & $\times1$ & $\times5$ \\
& DDU            & 80.45 & 0.031 & 0.034 & $\times1$ & -- \\
& TTA            & \textbf{83.00} & 0.058 & \underline{0.031} & $\times1$ & $\times5$ \\
& Masksembles    & 81.06 & 0.029 & 0.035 & $\times1$ & $\times5$ \\
& Deep Ensemble  & 81.72 & \underline{0.017} & \underline{0.031} & $\times5$ & $\times5$ \\
& Laplace Redux  & 79.56 & 0.086 & 0.043 & $\times1$ & -- \\
& Iterative      & 79.91 & 0.024 & 0.035 & $\times3$ & $\times3$ \\
& \cellcolor{oursgreen}\textbf{Ours}
& \cellcolor{oursgreen}\underline{81.80}
& \cellcolor{oursgreen}\textbf{0.015}
& \cellcolor{oursgreen}\textbf{0.029}
& \cellcolor{oursgreen}$\times1$
& \cellcolor{oursgreen}$\times5$ \\
& \cellcolor{oursgreen}
& \cellcolor{oursgreen}$\pm$1.43
& \cellcolor{oursgreen}$\pm$0.001
& \cellcolor{oursgreen}$\pm$0.007
& \cellcolor{oursgreen}
& \cellcolor{oursgreen} \\
\hline

\multirow{9}{*}{LiTS}
& MC Dropout     & 52.28 & 0.004 & 0.003 & $\times1$ & $\times5$ \\
& DDU            & 52.55 & 0.004 & 0.003 & $\times1$ & -- \\
& TTA            & 52.41 & 0.004 & 0.003 & $\times1$ & $\times5$ \\
& Masksembles    & 52.80 & \textbf{0.001} & 0.003 & $\times1$ & $\times5$ \\
& Deep Ensemble  & 53.30 & \textbf{0.001} & \textbf{0.001} & $\times5$ & $\times5$ \\
& Laplace Redux  & 52.35 & 0.003 & 0.003 & $\times1$ & -- \\
& Iterative      & \underline{53.34} & \underline{0.002} & \underline{0.002} & $\times3$ & $\times3$ \\
& \cellcolor{oursgreen}\textbf{Ours}
& \cellcolor{oursgreen}\textbf{54.22}
& \cellcolor{oursgreen}\textbf{0.001}
& \cellcolor{oursgreen}\underline{0.002}
& \cellcolor{oursgreen}$\times1$
& \cellcolor{oursgreen}$\times5$ \\
& \cellcolor{oursgreen}
& \cellcolor{oursgreen}$\pm$0.64
& \cellcolor{oursgreen}$\pm$0.000
& \cellcolor{oursgreen}$\pm$0.000
& \cellcolor{oursgreen}
& \cellcolor{oursgreen} \\
\hline

\multirow{9}{*}{CREMI}
& MC Dropout     & 37.23 & 0.035 & 0.017 & $\times1$ & $\times5$ \\
& DDU            & 34.12 & 0.016 & 0.018 & $\times1$ & -- \\
& TTA            & 38.70 & 0.011 & 0.016 & $\times1$ & $\times5$ \\
& Masksembles    & 36.97 & 0.014 & 0.017 & $\times1$ & $\times5$ \\
& Deep Ensemble  & 38.49 & \underline{0.009} & \underline{0.014} & $\times5$ & $\times5$ \\
& Laplace Redux  & 35.47 & 0.011 & 0.015 & $\times1$ & -- \\
& Iterative      & \underline{39.48} & 0.050 & 0.023 & $\times3$ & $\times3$ \\
& \cellcolor{oursgreen}\textbf{Ours}
& \cellcolor{oursgreen}\textbf{42.58}
& \cellcolor{oursgreen}\textbf{0.004}
& \cellcolor{oursgreen}\textbf{0.009}
& \cellcolor{oursgreen}$\times1$
& \cellcolor{oursgreen}$\times5$ \\
& \cellcolor{oursgreen}
& \cellcolor{oursgreen}$\pm$1.80
& \cellcolor{oursgreen}$\pm$0.002
& \cellcolor{oursgreen}$\pm$0.004
& \cellcolor{oursgreen}
& \cellcolor{oursgreen} \\
\hline

\end{tabular}%
}
\end{center}
\caption{Segmentation uncertainty, calibration, and relative computational cost.}
\label{tab:seg_uncertainty}
\end{table}

\begin{table}[!t]
\begin{center}
\resizebox{\columnwidth}{!}{%
\begin{tabular}{|l|c|c|c|c|c|c|}
\hline
\multirow{2}{*}{\textbf{Method}} 
& \multicolumn{2}{c|}{\textbf{Mass. $\rightarrow$ RoadTracer}} 
& \multicolumn{2}{c|}{\textbf{DRIVE $\rightarrow$ CHASE}} 
& \multicolumn{2}{c|}{\textbf{CREMI $\rightarrow$ ISBI}} \\
\cline{2-7}
& \textbf{AUROC (↑)} & \textbf{AUPR (↑)}
& \textbf{AUROC (↑)} & \textbf{AUPR (↑)}
& \textbf{AUROC (↑)} & \textbf{AUPR (↑)} \\
\hline\hline

MC Dropout
& 53.07 & 48.51
& 82.21 & 53.07
& 87.34 & 94.64 \\

DDU
& 54.63 & 52.54
& 85.99 & 61.77
& 88.79 & 94.82 \\

TTA
& 55.05 & \underline{56.62}
& \underline{89.94} & \textbf{67.31}
& \underline{93.10} & \underline{94.88} \\

Masksembles
& 54.00 & 51.91
& 86.35 & 64.51
& 87.72 & 94.14 \\

Deep Ensemble
& \textbf{57.32} & \textbf{57.85}
& 89.56 & \underline{66.33}
& 85.73 & 94.69 \\

Laplace Redux
& 52.82 & 52.04
& 86.46 & 57.90
& 88.68 & 94.77 \\

Iterative
& \underline{57.04} & 49.68
& 83.23 & 57.32
& 88.46 & 94.53 \\

\cellcolor{oursgreen}\textbf{Ours}
& \cellcolor{oursgreen}56.74
& \cellcolor{oursgreen}56.00
& \cellcolor{oursgreen}\textbf{90.52}
& \cellcolor{oursgreen}66.31
& \cellcolor{oursgreen}\textbf{95.43}
& \cellcolor{oursgreen}\textbf{95.68} \\

\cellcolor{oursgreen}
& \cellcolor{oursgreen}$\pm$0.71
& \cellcolor{oursgreen}$\pm$1.36
& \cellcolor{oursgreen}$\pm$0.47
& \cellcolor{oursgreen}$\pm$0.19
& \cellcolor{oursgreen}$\pm$0.24
& \cellcolor{oursgreen}$\pm$0.30 \\
\hline

\end{tabular}%
}
\end{center}
\caption{Cross-dataset segmentation OOD detection.}
\label{tab:seg_ood}
\end{table}

\begin{table}[!h]
\begin{center}
\resizebox{\columnwidth}{!}{%
\begin{tabular}{|l|l|c|c|c|c|c|c|}
\hline
\textbf{Dataset} & \textbf{Method} 
& \textbf{Acc (↑)} & \textbf{ECE (↓)} & \textbf{Brier (↓)}
& \textbf{AUC (↑)} & \textbf{AUPR (↑)} & \textbf{FPR95 (↓)} \\
\hline\hline

\multirow{8}{*}{CIFAR-10}
& MC Dropout     & 93.54 & \underline{0.011} & 0.094 & 84.22 & 87.36 & 27.10 \\
& DDU            & 94.12 & 0.018 & 0.090 & 85.20 & 89.26 & 24.11 \\
& TTA            & 94.23 & \underline{0.011} & 0.083 & \underline{93.64} & \underline{97.29} & 17.97 \\
& Masksembles    & \underline{94.61} & \underline{0.011} & \underline{0.080} & 91.69 & 92.88 & 18.47 \\
& Deep Ensemble  & \textbf{94.94} & \textbf{0.009} & \textbf{0.074} & 91.92 & 93.15 & \textbf{14.70} \\
& Laplace Redux  & 93.83 & 0.090 & 0.103 & 92.60 & 94.98 & 18.83 \\
& \cellcolor{oursgreen}\textbf{Ours}
& \cellcolor{oursgreen}94.59
& \cellcolor{oursgreen}\textbf{0.009}
& \cellcolor{oursgreen}0.083
& \cellcolor{oursgreen}\textbf{95.14}
& \cellcolor{oursgreen}\textbf{97.54}
& \cellcolor{oursgreen}\underline{16.36} \\
& \cellcolor{oursgreen}
& \cellcolor{oursgreen}$\pm$0.31
& \cellcolor{oursgreen}$\pm$0.001
& \cellcolor{oursgreen}$\pm$0.005
& \cellcolor{oursgreen}$\pm$0.22
& \cellcolor{oursgreen}$\pm$0.35
& \cellcolor{oursgreen}$\pm$0.29 \\
\hline

\multirow{7}{*}{MiniImageNet}
& MC Dropout     & 64.70 & 0.067 & 0.212 & 81.57 & 86.96 & 53.94 \\
& DDU            & 65.83 & 0.050 & 0.191 & \underline{84.21} & 88.49 & 49.85 \\
& TTA            & \underline{68.52} & 0.044 & \textbf{0.171} & 79.10 & 84.41 & 54.28 \\
& Masksembles    & 66.24 & 0.048 & 0.194 & 83.52 & 88.06 & 48.35 \\
& Deep Ensemble  & \textbf{70.50} & \underline{0.039} & 0.187 & 84.13 & \underline{89.20} & \underline{47.98} \\
& Laplace Redux  & 64.48 & 0.052 & 0.215 & 80.39 & 86.14 & 54.55 \\
& \cellcolor{oursgreen}\textbf{Ours}
& \cellcolor{oursgreen}66.90
& \cellcolor{oursgreen}\textbf{0.031}
& \cellcolor{oursgreen}\underline{0.178}
& \cellcolor{oursgreen}\textbf{86.42}
& \cellcolor{oursgreen}\textbf{90.66}
& \cellcolor{oursgreen}\textbf{45.24} \\
\hline

\multirow{7}{*}{BloodMNIST}
& MC Dropout     & 98.42 & \textbf{0.008} & 0.027 & 93.81 & 96.12 & 22.36 \\
& DDU            & 98.52 & \underline{0.009} & \underline{0.023} & 98.42 & 98.86 & 9.38 \\
& TTA            & \underline{98.60} & 0.010 & 0.030 & 98.43 & 99.09 & 6.40 \\
& Masksembles    & 98.40 & 0.010 & 0.025 & 98.74 & 99.11 & 7.61 \\
& Deep Ensemble  & 98.45 & 0.010 & \textbf{0.020} & \textbf{99.06} & \underline{99.50} & \underline{5.11} \\
& Laplace Redux  & 80.35 & 0.011 & 0.029 & 94.23 & 97.16 & 7.61 \\
& \cellcolor{oursgreen}\textbf{Ours}
& \cellcolor{oursgreen}\textbf{98.72}
& \cellcolor{oursgreen}\textbf{0.008}
& \cellcolor{oursgreen}\underline{0.023}
& \cellcolor{oursgreen}\underline{98.98}
& \cellcolor{oursgreen}\textbf{99.53}
& \cellcolor{oursgreen}\textbf{4.50} \\
\hline

\multirow{7}{*}{PathMNIST}
& MC Dropout     & 94.14 & \underline{0.034} & 0.096 & \underline{89.43} & \underline{68.70} & 25.15 \\
& DDU            & 94.53 & 0.041 & 0.094 & 88.39 & 55.97 & 24.52 \\
& TTA            & 94.00 & 0.041 & 0.098 & 85.72 & 56.60 & \underline{20.98} \\
& Masksembles    & 94.74 & 0.037 & 0.098 & 85.36 & 57.25 & 23.38 \\
& Deep Ensemble  & \underline{95.01} & \textbf{0.018} & \textbf{0.076} & 86.04 & 58.20 & 22.09 \\
& Laplace Redux  & 93.77 & 0.045 & 0.111 & 44.68 & 27.51 & 92.09 \\
& \cellcolor{oursgreen}\textbf{Ours}
& \cellcolor{oursgreen}\textbf{95.40}
& \cellcolor{oursgreen}0.037
& \cellcolor{oursgreen}\underline{0.088}
& \cellcolor{oursgreen}\textbf{94.87}
& \cellcolor{oursgreen}\textbf{80.58}
& \cellcolor{oursgreen}\textbf{11.60} \\
\hline

\end{tabular}%
}
\end{center}
\caption{Classification accuracy and OOD detection on natural and medical benchmarks.}
\label{tab:classification_results}
\end{table}
\textbf{Classification.}
We evaluate classification uncertainty on natural and medical benchmarks. 
Table~\ref{tab:classification_results} reports classification performance and OOD detection. The results show that stochastic activation switching improves uncertainty behavior while maintaining strong classification performance. 
On the natural-image benchmarks, the method achieves strong calibration and OOD detection; on MiniImageNet, however, its classification accuracy is lower than Deep Ensembles (66.90 vs. 70.50), despite better calibration and OOD detection metrics.
On high-resolution medical benchmarks, method transfers to Swin Transformer V2, maintaining high in-distribution accuracy while improving cross-domain OOD separation.

\textbf{Overall.}
Stochastic activations improves uncertainty quality while maintaining competitive task performance across most settings, with MiniImageNet the main exception in.
The method performs well across visual domains and architectures, including U-Net, ResNet-style classifiers, and Swin Transformer V2.
These results support the main claim that useful predictive diversity can be obtained from a single trained model without changing the input, adding an uncertainty head, redesigning the architecture, or training multiple models.

%% file: Sections/Ablation.tex
\section{Ablation Study}
\label{sec:ablation_analysis}

We analyze how the main design choices of stochastic activation switching affect model behavior. 
All ablations are conducted on CIFAR-10 with the same ResNet-style backbone and training protocol. 
We study four aspects: stochasticity granularity, the activation pair, information preservation, and inference-pass count.
These analyses explain not only whether the method works, but also why local activation-level stochasticity provides predictive diversity without substantially disrupting the learned representation.

\textbf{Granularity of Stochasticity.}
We compare layer-level, channel-level, and node-level stochastic activation sampling. 
Layer-level sampling shares one activation choice across an entire layer, channel-level sampling shares one choice within each feature channel, and node-level sampling independently samples the activation at each unit. 
As shown in Table~\ref{tab:ablation_granularity}, although layer-level stochasticity gives the lowest ECE, node-level sampling provides the strongest overall trade-off across predictive performance and uncertainty quality, improving accuracy, NLL \cite{guo2017calibration}, and Brier score.
This indicates that injecting stochasticity at the most local scale creates richer functional diversity inside the network. 
This micro-level variation improves predictive uncertainty estimation while reducing local activation co-adaptation.

\begin{table}[h]
\begin{center}
\begin{tabular}{|l|c|c|c|c|}
\hline
\textbf{Granularity} & \textbf{Acc. (↑)} & \textbf{NLL (↓)} & \textbf{ECE (↓)} & \textbf{Brier (↓)} \\
\hline\hline
Layer-level   & 93.04 & 0.204 & \textbf{0.005} & 0.101 \\
Channel-level & \underline{93.66} & 0.200 & 0.009 & 0.094 \\
Node-level    & \textbf{94.68} & \textbf{0.162} & 0.009 & \textbf{0.088} \\
\hline
\end{tabular}
\end{center}
\caption{Ablation on stochastic activation granularity.}
\label{tab:ablation_granularity}
\end{table}

\textbf{Choice of Stochastic Activation.}
We pair the backbone's native activation with one alternative and ablate this choice in Table~\ref{tab:ablation_activation}.
ReLU+SiLU gives the best accuracy and NLL, whereas ReLU+Tanh gives the best ECE and Brier, indicating that smoother alternatives favor task performance while stronger nonlinear changes improve calibration.
Using more than two candidate activations was less stable, likely due to excessive feature-distribution variation, so we use two-activation switching.
We do not tune activations per dataset: SiLU is fixed for classification and Tanh for segmentation across datasets, modalities, and architectures.
Other pairs remain competitive, suggesting that per-dataset tuning is unnecessary.

\begin{table}[!h]
\begin{center}
\begin{tabular}{|l|c|c|c|c|}
\hline
\textbf{Activation Pair} & \textbf{Acc. (↑)} & \textbf{NLL (↓)} & \textbf{ECE (↓)} & \textbf{Brier (↓)} \\
\hline\hline
ReLU + Tanh~\cite{lecun2002efficient} & \underline{94.43} & \underline{0.176} & \textbf{0.008} & \textbf{0.085} \\
ReLU + SiLU~\cite{elfwing2018sigmoid} & \textbf{94.68} & \textbf{0.162} & \underline{0.009} & \underline{0.088} \\
ReLU + GELU~\cite{hendrycks2016gaussian} & 94.12 & 0.184 & 0.012 & 0.095 \\
ReLU + Mish~\cite{misra2019mish} & 93.98 & 0.192 & 0.014 & 0.101 \\
ReLU + ELU~\cite{clevert2015fast}  & 93.67 & 0.206 & 0.018 & 0.109 \\
\hline
\end{tabular}
\end{center}
\caption{Effect of the activation pair used for stochastic switching.}
\label{tab:ablation_activation}
\end{table}

\textbf{Information Destruction Analysis.} 
Unlike masking-based uncertainty methods, stochastic activation switching perturbs the nonlinear transformation while keeping feature pathways active. 
As shown in Fig.~\ref{fig:idr_comparison}, our method gives lower IDR than MC Dropout across datasets, indicating that stochastic activation switching preserves intermediate representations under stochastic inference. 
This supports our motivation that predictive diversity can be induced through activation transformations without removing feature information.

\begin{figure}[h]
\begin{center}
\includegraphics[width=0.70\columnwidth]{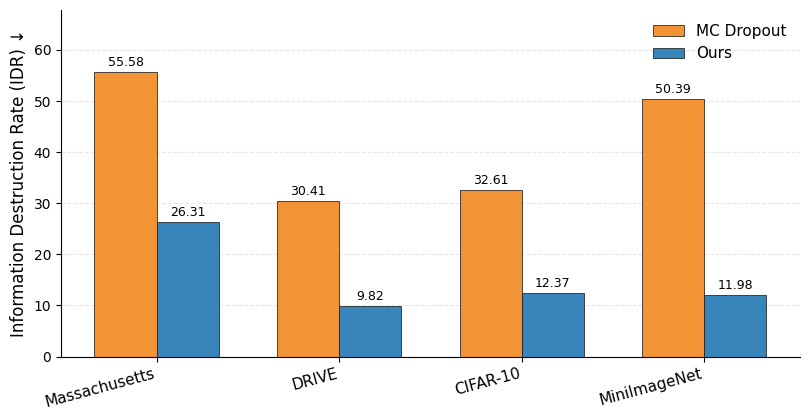}
\end{center}
\caption{
IDR comparison across datasets. Lower is better.
}
\label{fig:idr_comparison}
\end{figure}

\textbf{Number of inference passes.}
Finally, we study the effect of the number of stochastic forward passes used at inference time. 
Increasing \(T\) makes the predictive mean and uncertainty estimate more stable, but also increases inference time approximately linearly. 
Table~\ref{tab:ablation_passes} shows that uncertainty metrics improve as more passes are used, while accuracy saturates after a small number of samples. 
We therefore use \(T=5\) in the main experiments, as it provides a practical balance between uncertainty quality and computational cost.

\begin{table}[h]
\begin{center}
\renewcommand{\arraystretch}{0.85}
\begin{tabular}{|c|c|c|c|c|c|}
\hline
\textbf{Passes \(T\)} 
& \textbf{Accuracy (↑)} 
& \textbf{NLL (↓)} 
& \textbf{ECE (↓)} 
& \textbf{Brier (↓)} 
& \textbf{Time (s)} \\
\hline\hline
1  & 93.81 & 0.223 & 0.026 & 0.096 & 3.11 \\
3  & 94.26 & 0.186 & 0.013 & 0.085 & 7.06 \\
5  & 94.68 & 0.162 & 0.009 & 0.088 & 11.06 \\
10 & 94.37 & 0.176 & 0.007 & 0.082 & 20.99 \\
\hline
\end{tabular}%
\end{center}
\caption{Effect of the number of inference passes on acc, uncertainty quality, and runtime.}
\label{tab:ablation_passes}
\end{table}

%% file: Sections/Conclusion.tex
\section{Conclusion}
\label{sec:conclusion}

We reframe stochastic activation switching as an uncertainty estimation mechanism. Instead of relying on feature masking, input perturbations, or explicit ensembles, the method induces predictive diversity by sampling between nonlinear activation functions while keeping intermediate feature pathways active. Across classification, segmentation, and out-of-distribution detection experiments, the proposed approach improves uncertainty quality without degrading the main task performance, and remains effective across different visual domains and architectures. The results show that stochastic nonlinearities can provide a lightweight and information-preserving source of predictive variation from a single trained model, without requiring separate uncertainty heads or task-specific architectural redesign. A limitation is that the activation pair influences the trade-off between predictive performance and calibration, while repeated stochastic passes introduce additional inference cost. Future work can explore adaptive activation selection and more efficient sampling strategies to reduce inference overhead while preserving uncertainty estimation quality.


\newpage

%% file: Sections/supp.tex
\definecolor{oursgreen}{RGB}{220,245,220}
\definecolor{bmvcblue}{RGB}{0,45,120}
\definecolor{BMVCBLUE}{RGB}{0,45,120}

\runninghead{\textcolor{bmvcblue}{Sarhangzadeh \etal}}
{\textcolor{bmvcblue}{Supplementary Material}}

\setcounter{section}{0}
\renewcommand{\thesection}{\Alph{section}}
\renewcommand{\thesubsection}{\thesection.\arabic{subsection}}
\renewcommand{\theHsection}{supp.\arabic{section}}

\makeatletter
\renewcommand{\section}{%
  \@startsection{section}{1}{\z@}%
  {2.0ex plus 0.5ex minus 0.2ex}%
  {1.0ex plus 0.2ex}%
  {\normalfont\large\bfseries\color{bmvcblue}}%
}

\renewcommand{\subsection}{%
  \@startsection{subsection}{2}{\z@}%
  {1.5ex plus 0.4ex minus 0.2ex}%
  {0.8ex plus 0.2ex}%
  {\normalfont\normalsize\bfseries\color{bmvcblue}}%
}
\makeatother

\newtheorem{lemma}{Lemma}[section]
\newtheorem{theorem}{Theorem}[section]
\newtheorem{corollary}{Corollary}[section]
\newtheorem{remark}{Remark}[section]


\begin{center}
{\LARGE\bfseries\color{bmvcblue} Supplementary Material}
\end{center}

\section{Approach Overview}
\label{sec:supp_overview}

Stochastic Activation Switching replaces each deterministic activation in the backbone during training and inference with a per-unit Bernoulli switch between the backbone's native activation \(\sigma_{\mathrm{orig}}\) and a fixed alternative \(\sigma_{\mathrm{alt}}\).
During training, a fresh activation configuration is sampled at every forward pass, and the network is optimized with the standard task loss, so that the learned weights become robust to activation-level variation rather than to a single fixed nonlinearity.
At inference time, the trained weights are kept fixed and \(T\) stochastic forward passes are drawn from the same input; the empirical mean of the resulting predictions is used as the point estimate, and the empirical variance across passes is used as the uncertainty estimate.
Because every activation pathway remains active in all passes, stochasticity changes only the shape of the nonlinearity applied at each unit, without removing features from the intermediate representation.
This perturbation-without-removal property is the qualitative distinction from dropout-style mechanisms \cite{gal2016dropout} and the basis for the uncertainty estimator analyzed in Section~\ref{sec:supp_math}.

\section{Implementation Details}
\label{sec:supp_implementation}

\paragraph{Activation pair.}
For classification, we retain the backbone's native activation as
\(\sigma_{\mathrm{orig}}\) and use \(\sigma_{\mathrm{alt}}=\mathrm{SiLU}\)~\cite{elfwing2018sigmoid}.
For segmentation, we likewise retain the native activation and use
\(\sigma_{\mathrm{alt}}=\tanh\).
These fixed task-level choices are used across datasets and architectures, following the activation-choice ablation in Table~6 of the main paper.

\paragraph{Granularity and sampling rate.}
We use node-level granularity with an independent mask per activation unit and \(p=0.5\); both choices follow the granularity ablation in Table~5 of the main paper.
The choice \(p=0.5\) maximizes the entropy of the Bernoulli switch, giving both activations equal probability and avoiding a systematic preference for either nonlinearity.
Channel- and layer-level variants share a single mask across a channel or an entire layer and are not used in the main results.

\paragraph{Training.}
Training follows the standard supervised pipeline of the corresponding deterministic baseline for each dataset.
The optimizer, learning rate schedule, batch size, augmentation, and number of epochs are kept identical to those of the deterministic backbone.
The only modification is that each activation block performs Bernoulli switching at every forward pass during both training and inference.

\paragraph{Inference.}
At inference time, the trained weights are kept fixed and \(T=5\) stochastic forward passes are drawn per input, following the inference-pass ablation in Table~7 of the main paper.
The empirical mean is used as the prediction, and the empirical variance across the \(T\) passes is used as the uncertainty estimate, as defined in Equations~(7)--(8) of the main paper.

\paragraph{Backbones.}
For segmentation, we use a U-Net architecture~\cite{ronneberger2015unetconvolutionalnetworksbiomedical} on Massachusetts~\cite{MnihThesis}, DRIVE~\cite{staal2004ridge}, and LiTS~\cite{bilic2023liver}, and Swin Transformer V2~\cite{liu2022swin} on CREMI~\cite{cardona2010drosophila}.
For natural-image classification, we use a ResNet-style~\cite{he2016deep} convolutional backbone on CIFAR-10~\cite{krizhevsky2009learning} and MiniImageNet~\cite{vinyals2016matching}.
For high-resolution medical classification, we use Swin Transformer V2 on BloodMNIST and PathMNIST~\cite{yang2023medmnist}.
In all cases, the backbone matches the corresponding deterministic baseline, and only the activation block is replaced by the stochastic switch.

\paragraph{Datasets and splits.}
For segmentation, Massachusetts Roads is used with the official train-validation-test split, and RoadTracer~\cite{bastani2018roadtracerautomaticextractionroad} serves as the held-out cross-dataset OOD evaluation set.
DRIVE is used with the standard train/test split, and CHASE-DB1~\cite{fraz2012ensemble} is used for cross-dataset OOD evaluation.
LiTS axial slices are used as 2D inputs.
CREMI is used with its published volume split, and ISBI~2012~\cite{funke2018large} is used for cross-dataset OOD evaluation.
For classification, CIFAR-10 uses the canonical 50{,}000/10{,}000 train/test split, with SVHN~\cite{netzer2011reading} serving as the OOD set.
MiniImageNet is split following the standard 600 images per class across 100 classes, with CIFAR-100 used as the OOD set.
BloodMNIST and PathMNIST follow the MedMNIST~v2 official splits at \(224\times224\) resolution and are treated as a cross-domain OOD pair for each other.

\paragraph{Augmentation.}
For segmentation training, inputs are augmented with random crops, random horizontal and vertical flips, and \(90^\circ\) rotations; no augmentation is applied at test time.
For natural-image classification, we apply random crops with 4-pixel padding and random horizontal flips.
For high-resolution medical classification, we apply random resized crops and horizontal flips.
All inputs are channel-normalized using dataset-specific means and standard deviations.

\paragraph{Training hyperparameters.}
Table~\ref{tab:supp_hyperparams} summarizes the optimizer, learning rate, batch size, and number of epochs used for each backbone.
All models are trained from scratch with the standard supervised cross-entropy loss for classification and a combined soft-Dice plus binary cross-entropy loss for segmentation.
The optimizer and schedule are identical for the deterministic baseline and for Stochastic Activation Switching; only the activation block differs.

\begin{table}[h]
\centering
\small
\begin{tabular}{lllrlr}
\hline
Backbone & Dataset & Optimizer & Batch & LR / schedule & Epochs \\
\hline
U-Net      & Massachusetts & Adam \cite{diederik2014adam}              & 32  & \(1{\times}10^{-3}\), cosine    & 200 \\
U-Net      & DRIVE         & Adam              & 8   & \(1{\times}10^{-3}\), cosine    & 200 \\
U-Net      & LiTS          & Adam              & 64   & \(1{\times}10^{-4}\), cosine    & 200 \\
Swin-V2    & CREMI         & AdamW \cite{loshchilov2017decoupled}             & 32   & \(1{\times}10^{-4}\) & 200 \\
ResNet-18  & CIFAR-10      & SGD (\(m{=}0.9\)) & 256 & \(0.1\), cosine                 & 200 \\
ResNet-50  & MiniImageNet  & SGD (\(m{=}0.9\)) & 128 & \(0.1\), cosine                 & 200 \\
Swin-V2    & BloodMNIST    & AdamW             & 128  & \(1{\times}10^{-4}\) & 200 \\
Swin-V2    & PathMNIST     & AdamW             & 256  & \(1{\times}10^{-4}\) & 200 \\
\hline
\end{tabular}
\caption{Per-dataset training hyperparameters used for all methods.
Weight decay is \(5{\times}10^{-4}\) for SGD and \(1{\times}10^{-4}\) for AdamW.
}
\label{tab:supp_hyperparams}
\end{table}

\paragraph{Baseline configurations.}
MC Dropout~\cite{gal2016dropout} uses a dropout rate of \(p_d=0.2\) inserted after each activation block and kept active at inference with \(T=5\) stochastic forward passes.
DDU~\cite{mukhoti2023deep} applies post-hoc density estimation to features from the trained deterministic backbone.
Masksembles~\cite{durasov2021masksembles} uses five fixed partially overlapping masks, producing five predictions at inference.
Deep 
Deep Ensembles~\cite{lakshminarayanan2017simple} are formed from five independently trained models with different random seeds; the prediction is the per-class probability average over members and the uncertainty is the empirical variance across members.
Laplace Redux~\cite{daxberger2021laplace} applies a last-layer Laplace approximation to the trained backbone with a diagonal covariance and prior precision tuned on the validation split.
Test-Time Augmentation~\cite{wang2019aleatoric} samples five augmented versions of each test input (random crops and horizontal flips for classification; the segmentation augmentation pipeline above for dense prediction) and averages the predictions.
The iterative segmentation baseline~\cite{durasov2024enabling} uses three refinement iterations following the configuration of its released reference implementation.

\paragraph{Compute.}
All experiments were run with PyTorch on a single NVIDIA H100 GPU.
Inference with \(T=5\) stochastic forward passes increases the per-input inference cost by approximately a factor of five over the deterministic backbone, as reported in Table~7 of the main paper.
All results in this work are reported from a single random seed per method, except for the five-member Deep Ensemble baseline.

\section{Mathematical Analysis of Stochastic Activation Switching}
\label{sec:supp_math}

We provide a mathematical characterization of stochastic activation switching. 
We show that the method defines a well-posed stochastic predictor, derive closed-form conditional moments for a single stochastic activation unit, bound the induced per-unit variance under a bounded activation gap, and analyze the convergence of the empirical Monte Carlo estimators used at inference time.

\subsection{Notation and Setup}
\label{sec:supp_notation}

Fix a neural network \(f_\theta : \mathcal{X} \to \mathcal{Y}\) with parameters \(\theta\), whose nonlinear blocks contain activation units indexed by layer \(l\) and unit \(j\). 
Let \(z_j^{(l,t)}\) denote the pre-activation of unit \(j\) in layer \(l\) during stochastic forward pass \(t\). 
Let \(\sigma_{\mathrm{orig}},\sigma_{\mathrm{alt}}:\mathbb{R}\to\mathbb{R}\) be the original and alternative scalar activation functions. 
At pass \(t\in\{1,\ldots,T\}\) and unit \((l,j)\), draw an independent Bernoulli mask
\[
m_j^{(l,t)} \sim \mathrm{Bern}(p),
\qquad p\in(0,1),
\]
and define the stochastic activation output as
\begin{equation}
h^{(l,t)}_j
=
\bigl(1-m_j^{(l,t)}\bigr)
\sigma_{\mathrm{orig}}\!\bigl(z_j^{(l,t)}\bigr)
+
m_j^{(l,t)}
\sigma_{\mathrm{alt}}\!\bigl(z_j^{(l,t)}\bigr).
\label{eq:sas_unit}
\end{equation}
This is the same Bernoulli activation-switching rule used in the main paper. 
Unless stated otherwise, we analyze the node-level case, where each unit samples its own mask. 
For channel-level or layer-level sampling, the same mask is shared across a channel or across an entire layer.

In the following derivations, expectations are taken with the current pre-activation \(z_j^{(l,t)}\) fixed. 
This is important because, in deeper layers, \(z_j^{(l,t)}\) is itself affected by random activation choices made in previous layers. 
Thus, we first analyze the local randomness caused by the current mask, and full-network expectations can then be obtained by averaging over the upstream randomness as well. 
For readability, we omit layer and pass indices when the meaning is clear.

\subsection{The Stochastic Predictor and Empirical Estimators}
\label{sec:supp_predictor}

Let
\[
\sigma^{(t)} = \bigl(m_j^{(l,t)}\bigr)_{l,j}
\]
denote the collection of all activation-switching masks sampled during stochastic forward pass \(t\). 
With the trained weights \(\theta\) fixed, the output of the network for input \(x\) is
\begin{equation}
\eta^{(t)} = f_\theta\!\left(x;\sigma^{(t)}\right).
\label{eq:supp_stochastic_prediction}
\end{equation}
Thus, each stochastic pass produces a random prediction only through the sampled activation configuration. 
If the original and alternative activations differ for at least one unit that affects the output, then different mask samples can produce different predictions. 
If the two activations are identical, the stochastic predictor reduces to the deterministic network.

Over \(T\) stochastic forward passes, the prediction trajectory is
\[
\left\{\eta^{(1)},\eta^{(2)},\dots,\eta^{(T)}\right\}.
\]
Following the main paper, we estimate the predictive mean and uncertainty by
\begin{equation}
\hat{\mu}_T(x)
=
\frac{1}{T}\sum_{t=1}^{T}\eta^{(t)},
\qquad
\hat{U}_T(x)
=
\frac{1}{T}\sum_{t=1}^{T}
\left(\eta^{(t)}-\hat{\mu}_T(x)\right)^2.
\label{eq:supp_empirical_estimators}
\end{equation}
Here, \(\hat{\mu}_T(x)\) is the Monte Carlo estimate of the expected prediction, and \(\hat{U}_T(x)\) is the empirical variance across stochastic predictions. 
Their population counterparts are
\begin{equation}
\mu_\infty(x)
=
\mathbb{E}\!\left[\eta^{(t)}\right],
\qquad
V_\infty(x)
=
\mathrm{Var}\!\left[\eta^{(t)}\right],
\label{eq:supp_population_moments}
\end{equation}
which correspond to the ideal mean and variance under infinitely many stochastic activation samples.

\subsection{Per-Unit Moments and Bounded Activation Gap}
\label{sec:supp_unit_moments}

We first analyze the stochasticity introduced at a single activation unit. 
Conditioned on a fixed pre-activation value \(z\), the only source of randomness is the Bernoulli switching variable \(m\). 
Thus, the unit output is a two-point random variable: it takes the original activation value with probability \(1-p\), and the alternative activation value with probability \(p\).

\begin{lemma}[Unit-wise moments]
\label{lem:moments}
For any fixed \(z\in\mathbb{R}\), the stochastic activation output \(h\) in Equation~\eqref{eq:sas_unit} satisfies
\begin{equation}
\mathbb{E}[h\mid z]
=
(1-p)\sigma_{\mathrm{orig}}(z)
+
p\,\sigma_{\mathrm{alt}}(z),
\label{eq:supp_unit_mean}
\end{equation}
and
\begin{equation}
\mathrm{Var}[h\mid z]
=
p(1-p)
\left(
\sigma_{\mathrm{alt}}(z)
-
\sigma_{\mathrm{orig}}(z)
\right)^2.
\label{eq:supp_unit_variance}
\end{equation}
Therefore, the per-unit variance is zero whenever the two activation functions produce the same value at \(z\). 
In other words, stochastic activation switching creates local uncertainty only where the activation pair disagrees.
\end{lemma}

\begin{proof}
Since \(m\sim\mathrm{Bern}(p)\), we have \(\mathbb{E}[m]=p\). 
Taking expectation in Equation~\eqref{eq:sas_unit} directly gives Equation~\eqref{eq:supp_unit_mean}. 
For the second moment, because \(m\in\{0,1\}\), we have \(m^2=m\). 
Thus,
\[
\mathbb{E}[h^2\mid z]
=
(1-p)\sigma_{\mathrm{orig}}^2(z)
+
p\,\sigma_{\mathrm{alt}}^2(z).
\]
Subtracting \(\bigl(\mathbb{E}[h\mid z]\bigr)^2\) gives Equation~\eqref{eq:supp_unit_variance}.
\end{proof}

Lemma~\ref{lem:moments} shows that the amount of stochasticity introduced by one unit is controlled by two quantities: the switching probability \(p\), and the gap between the two activation functions at the current pre-activation value. 
If this activation gap is bounded, then the induced per-unit variance is also bounded.

\begin{theorem}[Bounded noise envelope]
\label{thm:bounded}
Let
\begin{equation}
\Delta
=
\sup_{z\in\mathbb{R}}
\left|
\sigma_{\mathrm{alt}}(z)
-
\sigma_{\mathrm{orig}}(z)
\right|.
\label{eq:supp_activation_gap}
\end{equation}
If \(\Delta<\infty\), then for every \(z\in\mathbb{R}\),
\begin{equation}
\mathrm{Var}[h\mid z]
\le
p(1-p)\Delta^2
\le
\frac{1}{4}\Delta^2.
\label{eq:supp_bounded_variance}
\end{equation}
The upper bound is tight at \(p=\frac{1}{2}\).
\end{theorem}

\begin{proof}
From Lemma~\ref{lem:moments},
\[
\mathrm{Var}[h\mid z]
=
p(1-p)
\left(
\sigma_{\mathrm{alt}}(z)
-
\sigma_{\mathrm{orig}}(z)
\right)^2.
\]
By definition of \(\Delta\),
\[
\left(
\sigma_{\mathrm{alt}}(z)
-
\sigma_{\mathrm{orig}}(z)
\right)^2
\le
\Delta^2.
\]
Also, for \(p\in[0,1]\), we have \(p(1-p)\le \frac{1}{4}\), with equality at \(p=\frac{1}{2}\). 
Combining these two facts proves the result.
\end{proof}

\begin{remark}[Per-unit bound only]
\label{rem:per-unit}
The bound in Theorem~\ref{thm:bounded} applies to a single activation unit. 
It does not directly imply a bound on the full network prediction variance, because per-unit perturbations can be amplified, damped, or combined as they propagate through later layers.
\end{remark}

For activation pairs with unbounded global gap, such as ReLU--Tanh, the same argument gives a bound on any bounded pre-activation range. 
In particular, if \(z\in[-R,R]\), one may replace \(\Delta\) in Equation~\eqref{eq:supp_activation_gap} by
\[
\Delta_R
=
\sup_{z\in[-R,R]}
\left|
\sigma_{\mathrm{alt}}(z)
-
\sigma_{\mathrm{orig}}(z)
\right|,
\]
which is finite whenever the activation functions are finite on the interval.

For some activation pairs, the global gap \(\Delta\) can be computed explicitly. 
The following corollary gives this value for the ReLU--SiLU pair, which is used in our classification experiments.

\begin{corollary}[ReLU--SiLU activation-gap constant]
\label{cor:silu}
Let
\[
\sigma_{\mathrm{orig}}(z)
=
\mathrm{ReLU}(z)
=
\max(0,z),
\qquad
\sigma_{\mathrm{alt}}(z)
=
\mathrm{SiLU}(z)
=
z\,\mathrm{sigmoid}(z).
\]
Define the gap function
\begin{equation}
\phi(z)
=
\sigma_{\mathrm{alt}}(z)
-
\sigma_{\mathrm{orig}}(z).
\label{eq:supp_silu_gap}
\end{equation}
Then
\begin{equation}
\phi(z)
=
\begin{cases}
-z\,\mathrm{sigmoid}(-z), & z \ge 0,\\[3pt]
z\,\mathrm{sigmoid}(z), & z < 0.
\end{cases}
\label{eq:supp_silu_gap_piecewise}
\end{equation}
The function \(\phi\) is non-positive and even, meaning \(\phi(-a)=\phi(a)\). 
Moreover,
\begin{equation}
\Delta_{\mathrm{SiLU}}
=
\sup_{z\in\mathbb{R}}|\phi(z)|
=
z^\star \mathrm{sigmoid}(-z^\star)
\approx
0.27846,
\label{eq:supp_silu_delta}
\end{equation}
where \(z^\star\approx 1.27846\) is the unique positive solution of
\begin{equation}
e^{z^\star}(z^\star-1)=1,
\label{eq:supp_silu_root}
\end{equation}
equivalently,
\[
z^\star \mathrm{sigmoid}(z^\star)=1.
\]
\end{corollary}

\begin{proof}
For \(z\ge 0\), we have
\[
\phi(z)
=
z\,\mathrm{sigmoid}(z)-z
=
-z\bigl(1-\mathrm{sigmoid}(z)\bigr)
=
-z\,\mathrm{sigmoid}(-z).
\]
For \(z<0\), \(\mathrm{ReLU}(z)=0\), so
\[
\phi(z)
=
z\,\mathrm{sigmoid}(z).
\]
For any \(a>0\),
\[
\phi(a)
=
-a\,\mathrm{sigmoid}(-a),
\qquad
\phi(-a)
=
-a\,\mathrm{sigmoid}(-a),
\]
so \(\phi(-a)=\phi(a)\). 
Thus, the gap is symmetric around zero.

It remains to find the maximum magnitude. 
For \(z>0\),
\[
\phi(z)
=
-z\,\mathrm{sigmoid}(-z).
\]
Differentiating gives
\[
\phi'(z)
=
\mathrm{sigmoid}(-z)
\left(
z\,\mathrm{sigmoid}(z)-1
\right).
\]
Since \(\mathrm{sigmoid}(-z)>0\), the stationary point satisfies
\[
z\,\mathrm{sigmoid}(z)=1.
\]
This is equivalent to
\[
z=1+e^{-z}
\quad\Longleftrightarrow\quad
e^z(z-1)=1.
\]
Solving this equation gives \(z^\star\approx 1.27846\), and therefore
\[
|\phi(z^\star)|
=
z^\star \mathrm{sigmoid}(-z^\star)
\approx
0.27846.
\]
By symmetry, the negative side has the same maximum magnitude.
\end{proof}

\begin{corollary}[Numerical bound at \(p=\tfrac{1}{2}\)]
\label{cor:numeric}
For ReLU--SiLU switching with \(p=\frac{1}{2}\),
\begin{equation}
\sup_{z\in\mathbb{R}}
\mathrm{Var}[h\mid z]
=
\frac{1}{4}\Delta_{\mathrm{SiLU}}^2
\approx
0.01939.
\label{eq:supp_numeric_bound}
\end{equation}
\end{corollary}

\subsection{Convergence of the Empirical Estimators}
\label{sec:supp_convergence}

We now analyze the Monte Carlo estimators used at inference time. 
For a fixed input \(x\), each stochastic forward pass produces a random prediction \(\eta^{(t)}\). 
Throughout this subsection, we assume that masks are resampled independently across stochastic forward passes, so that \(\eta^{(1)},\ldots,\eta^{(T)}\) are i.i.d. draws from the stochastic predictor for the fixed input \(x\). 
The empirical mean \(\hat{\mu}_T(x)\) estimates the expected prediction, while the empirical variance \(\hat{U}_T(x)\) estimates the predictive variance induced by stochastic activation switching.

\begin{theorem}[Unbiasedness and concentration of \(\hat{\mu}_T\)]
\label{thm:mean}
The empirical mean is an unbiased estimator of the population mean:
\begin{equation}
\mathbb{E}[\hat{\mu}_T(x)] = \mu_\infty(x).
\label{eq:mean_unbiased}
\end{equation}
Assume the scalar output of interest is bounded almost surely, so that
\[
\eta^{(t)} \in [a,b]
\qquad
\text{for all } t.
\]
This assumption holds for post-softmax class probabilities, pixel-wise probabilities in semantic segmentation, and sigmoid-bounded scalar outputs, but not for unbounded logits without additional assumptions. 
Then, for any \(\delta>0\),
\begin{equation}
\Pr\!\left(
\left|\hat{\mu}_T(x)-\mu_\infty(x)\right|
\ge
\delta
\right)
\le
2\exp\!\left(
-\frac{2T\delta^2}{(b-a)^2}
\right).
\label{eq:hoeffding_mean}
\end{equation}
For a \(K\)-dimensional bounded output, applying Equation~\eqref{eq:hoeffding_mean} coordinate-wise and using a union bound replaces the leading constant \(2\) by \(2K\). 
For pixel-wise outputs, the same argument applies coordinate-wise to each pixel and class probability.
\end{theorem}

\begin{proof}
Unbiasedness follows from linearity of expectation:
\[
\mathbb{E}[\hat{\mu}_T(x)]
=
\mathbb{E}\!\left[
\frac{1}{T}\sum_{t=1}^{T}\eta^{(t)}
\right]
=
\frac{1}{T}\sum_{t=1}^{T}\mathbb{E}[\eta^{(t)}]
=
\mu_\infty(x).
\]
The concentration bound in Equation~\eqref{eq:hoeffding_mean} is Hoeffding's inequality applied to the i.i.d. bounded sequence \(\{\eta^{(t)}\}_{t=1}^{T}\). 
The \(K\)-dimensional version follows by applying the same bound to each coordinate and taking a union bound.
\end{proof}

The empirical variance \(\hat{U}_T(x)\) is computed with normalization \(1/T\), as in the main paper. 
As with the usual sample variance, this estimator is slightly biased for finite \(T\), but the bias vanishes as \(T\) increases.

\begin{lemma}[Bias of \(\hat{U}_T\)]
\label{lem:bias}
The empirical uncertainty estimator satisfies
\begin{equation}
\mathbb{E}[\hat{U}_T(x)]
=
\frac{T-1}{T}V_\infty(x).
\label{eq:variance_bias}
\end{equation}
Therefore, the Bessel-corrected estimator
\[
\hat{U}_T^{\mathrm{u}}
=
\frac{T}{T-1}\hat{U}_T
\]
is unbiased.
\end{lemma}

\begin{proof}
Let \(\eta\) denote a generic draw from the same distribution as each \(\eta^{(t)}\). 
Using the sample-variance identity,
\[
\sum_{t=1}^{T}
\left(
\eta^{(t)}-\hat{\mu}_T
\right)^2
=
\sum_{t=1}^{T}\eta^{(t)2}
-
T\hat{\mu}_T^2.
\]
Taking expectations gives
\[
\mathbb{E}\!\left[
\sum_{t=1}^{T}\eta^{(t)2}
\right]
=
T\mathbb{E}[\eta^2],
\]
and
\[
\mathbb{E}\!\left[T\hat{\mu}_T^2\right]
=
T\left(
\mu_\infty^2+\frac{V_\infty}{T}
\right).
\]
Dividing by \(T\) yields
\[
\mathbb{E}[\hat{U}_T]
=
V_\infty\left(1-\frac{1}{T}\right)
=
\frac{T-1}{T}V_\infty.
\]
\end{proof}

\begin{theorem}[Convergence of \(\hat{U}_T\)]
\label{thm:UT}
Under the bounded-output assumption of Theorem~\ref{thm:mean}, the empirical variance converges to the population predictive variance:
\begin{equation}
\hat{U}_T(x)
\to
V_\infty(x)
\quad
\text{in } L^2
\quad
\text{as } T\to\infty.
\label{eq:variance_convergence}
\end{equation}
Moreover, its root-mean-squared error is of order \(O(T^{-1/2})\).
\end{theorem}

\begin{proof}
We decompose the mean-squared error into squared bias and variance:
\begin{equation}
\mathbb{E}\!\left[
\left(
\hat{U}_T - V_\infty
\right)^2
\right]
=
\left(
\mathbb{E}[\hat{U}_T]-V_\infty
\right)^2
+
\mathrm{Var}(\hat{U}_T).
\label{eq:mse_decomposition}
\end{equation}

By Lemma~\ref{lem:bias}, the bias is
\[
\mathbb{E}[\hat{U}_T]-V_\infty
=
-\frac{V_\infty}{T}.
\]
Since the output is bounded in \([a,b]\), Popoviciu's inequality gives
\[
V_\infty
\le
\frac{(b-a)^2}{4}.
\]
Therefore, the squared bias is bounded by
\[
\frac{(b-a)^4}{16T^2}
=
O(T^{-2}).
\]

For the variance term, the Bessel-corrected estimator can be written as the degree-2 \(U\)-statistic
\begin{equation}
\hat{U}_T^{\mathrm{u}}
=
\frac{1}{T(T-1)}
\sum_{i\neq j}
\frac{1}{2}
\left(
\eta^{(i)}-\eta^{(j)}
\right)^2.
\label{eq:u_stat_variance}
\end{equation}
Its kernel is
\[
h(u,v)
=
\frac{1}{2}(u-v)^2.
\]
Because \(u,v\in[a,b]\), the kernel is bounded:
\[
h(u,v)
\in
\left[
0,\frac{(b-a)^2}{2}
\right].
\]
The standard variance bound for bounded \(U\)-statistics gives
\[
\mathrm{Var}\!\left(
\hat{U}_T^{\mathrm{u}}
\right)
=
O\!\left(
\frac{(b-a)^4}{T}
\right).
\]
Since
\[
\mathrm{Var}(\hat{U}_T)
=
\left(
\frac{T-1}{T}
\right)^2
\mathrm{Var}\!\left(
\hat{U}_T^{\mathrm{u}}
\right),
\]
the same \(O(T^{-1})\) rate holds for \(\mathrm{Var}(\hat{U}_T)\).

Combining the \(O(T^{-2})\) squared-bias term with the \(O(T^{-1})\) variance term gives
\[
\mathbb{E}\!\left[
\left(
\hat{U}_T - V_\infty
\right)^2
\right]
=
O(T^{-1}).
\]
Taking the square root gives root-mean-squared error \(O(T^{-1/2})\), proving \(L^2\) convergence.
\end{proof}

The constants in the above concentration and \(U\)-statistic bounds are conservative. 
In practice, the number of inference passes should be chosen by balancing uncertainty stability and computational cost, as studied in the inference-pass ablation of the main paper.

\section{Computational Cost and Timing Analysis}
\label{sec:supp_timing}

We analyze the computational cost of stochastic activation switching in terms of both training and inference time.
The analysis separates the overhead introduced by the stochastic activation mechanism itself from the cost of performing multiple stochastic forward passes.

\paragraph{Measurement protocol.}
All timings are measured on a single NVIDIA H100 GPU.
For CIFAR-10, training time corresponds to one training epoch and inference time corresponds to evaluation over the complete test set.
All methods use the same backbone and input setting.
Our method and MC Dropout use \(T=5\) stochastic forward passes, while Deep Ensembles use five independently trained models.
The reported times therefore capture the complete end-to-end cost of each uncertainty-estimation procedure rather than only the deterministic backbone execution.

\paragraph{Sampling overhead.}
To isolate the sampling overhead under the same architecture and input setting, we compare the deterministic model, MC Dropout, Deep Ensembles, and our method on CIFAR-10.
The deterministic model requires \(40.05\,\mathrm{s}\) for one training epoch and \(2.83\,\mathrm{s}\) for full test-set inference.
MC Dropout requires \(40.97\,\mathrm{s}\) and \(13.66\,\mathrm{s}\), respectively, while our method requires \(42.73\pm0.36\,\mathrm{s}\) for training and \(14.21\pm0.18\,\mathrm{s}\) for inference.
Thus, our method has training and inference costs close to MC Dropout under the same \(T=5\) sampling setting, with only a small additional overhead from stochastic activation switching.
A five-member Deep Ensemble requires \(164.52\,\mathrm{s}\) for training and \(13.27\,\mathrm{s}\) for inference.
Therefore, our method has a similar multi-pass inference cost to Deep Ensembles while requiring substantially less training and model storage, since only a single trained network is retained instead of five independently trained models.

\begin{table}[!h]
\begin{center}
\begin{tabular}{|l|c|c|}
\hline
\textbf{Method} & \textbf{Training Time (s) $\downarrow$} & \textbf{Inference Time (s) $\downarrow$} \\
\hline\hline
Deterministic & \textbf{40.05} & \textbf{2.83} \\
MC Dropout & \underline{40.97} & 13.66 \\
DDU & 49.93 & \underline{2.86} \\
TTA & \textbf{40.05} & 18.43 \\
Masksembles & 43.30 & 3.16 \\
Deep Ensemble & 194.52 & 13.27 \\
Laplace Redux & 57.05 & 3.07 \\
\rowcolor{oursgreen}
Ours & $42.73\pm0.36$ & $14.21\pm0.18$ \\
\hline
\end{tabular}
\end{center}
\caption{Training and inference time on CIFAR-10.}
\label{tab:supp_classification_timing}
\end{table}

\begin{table}[!h]
\begin{center}
\resizebox{\textwidth}{!}{
\begin{tabular}{|l|cc|cc|cc|cc|}
\hline
\multirow{2}{*}{\textbf{Method}}
& \multicolumn{2}{c|}{\textbf{Massachusetts}}
& \multicolumn{2}{c|}{\textbf{DRIVE}}
& \multicolumn{2}{c|}{\textbf{LiTS}}
& \multicolumn{2}{c|}{\textbf{CREMI}} \\
\cline{2-9}
& \textbf{Train} & \textbf{Infer.}
& \textbf{Train} & \textbf{Infer.}
& \textbf{Train} & \textbf{Infer.}
& \textbf{Train} & \textbf{Infer.} \\
\hline\hline

Deterministic
& \textbf{54.41} & \textbf{1.52}
& \underline{37.30} & \textbf{24.30}
& \textbf{28.27} & \textbf{10.98}
& \textbf{36.98} & 9.22 \\

MC Dropout
& \underline{54.54} & 4.46
& \textbf{38.28} & \underline{140.43}
& 32.92 & 55.78
& 39.27 & 52.90 \\

DDU
& 74.41 & 3.01
& 51.70 & 26.40
& 48.38 & \underline{11.09}
& 52.45 & \textbf{9.36} \\

Masksembles
& 57.91 & 1.90
& 39.40 & 29.60
& 33.73 & 12.51
& 37.64 & \underline{12.51} \\

Deep Ensemble
& 261.79 & \underline{4.25}
& 214.74 & 135.11
& 103.84 & 54.47
& 189.48 & 49.42 \\

\rowcolor{oursgreen}
Ours
& $58.46\pm3.44$ & $4.89\pm0.62$
& $38.80\pm2.70$ & $145.50\pm2.10$
& $\underline{32.10\pm0.77}$ & $59.69\pm0.58$
& $\underline{38.03\pm2.36}$ & $62.57\pm1.33$ \\

\hline
\end{tabular}}
\end{center}
\caption{Training and inference time on the segmentation benchmarks. All times are reported in seconds.}
\label{tab:supp_segmentation_timing}
\end{table}

Tables~\ref{tab:supp_classification_timing} and~\ref{tab:supp_segmentation_timing}
compare the wall-clock computational cost across classification and segmentation.
Training time corresponds to one complete epoch over the training set; for Deep Ensembles,
the reported value includes one epoch for each of the five independently trained members.
Inference time corresponds to processing the complete test set using the full inference
procedure of each method.

Across both tasks, stochastic activation switching retains a single-model training profile
and has training cost close to other single-model approaches, while avoiding the repeated
optimization and model storage required by Deep Ensembles.
Its main computational overhead occurs at inference because uncertainty estimation requires
multiple stochastic forward passes.
Accordingly, its inference cost is generally comparable to other multi-evaluation approaches,
including MC Dropout and Deep Ensembles, while TTA similarly incurs additional cost from
repeated augmented evaluations.
In contrast, deterministic inference, DDU, Laplace Redux, and Masksembles generally provide
lower wall-clock inference latency.

These results clarify that the main computational advantage of stochastic activation
switching lies in single-model training and storage rather than minimum inference latency.
The additional inference computation provides the predictive diversity used for uncertainty
estimation and accompanies the calibration and OOD-detection improvements reported in the
main paper.

\begin{figure}[t]
    \centering
    \includegraphics[width=0.85\columnwidth]{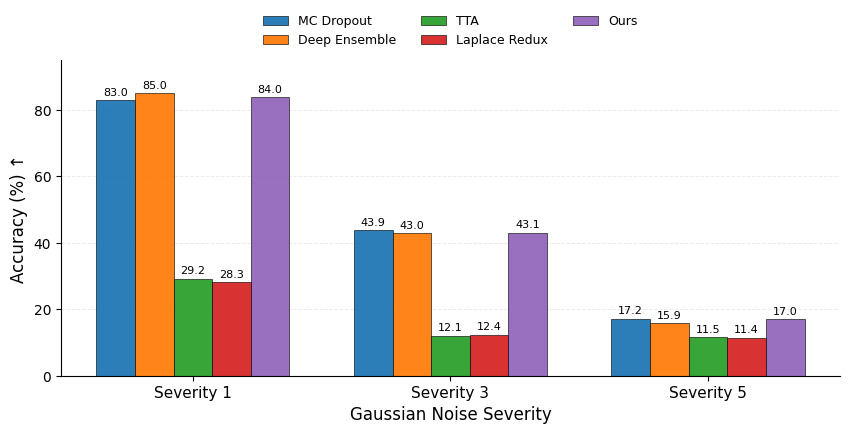}
    \caption{
    Accuracy under CIFAR-10-C Gaussian noise corruption at severities \(1\), \(3\), and \(5\).
    MC Dropout gives slightly higher accuracy under heavier corruption, likely because feature masking acts as a stronger noise-robust regularizer.
    Our method remains close in accuracy while preserving information flow through activation switching instead of feature removal.
    }
    \label{fig:cifar_c_gaussian_accuracy}
\end{figure}

\begin{table}[t]
\centering
\scriptsize
\setlength{\tabcolsep}{3pt}
\renewcommand{\arraystretch}{0.75}
\resizebox{0.75\columnwidth}{!}{%
\begin{tabular}{|c|l|c|c|c|c|}
\hline
\textbf{Severity} & \textbf{Method} & \textbf{Acc. (↑)} & \textbf{NLL (↓)} & \textbf{ECE (↓)} & \textbf{Brier (↓)} \\
\hline\hline

\multirow{5}{*}{1}
& MC Dropout     & 82.95 & \underline{0.517} & \textbf{0.015} & 0.243 \\
& Deep Ensemble  & \textbf{84.96} & \textbf{0.486} & 0.042 & \underline{0.220} \\
& TTA            & 29.20 & 3.030 & 0.401 & 1.019 \\
& Laplace Redux  & 28.29 & 2.315 & 0.213 & 0.872 \\
& \cellcolor{oursgreen}\textbf{Ours}
& \cellcolor{oursgreen}\underline{83.96}
& \cellcolor{oursgreen}0.521
& \cellcolor{oursgreen}\underline{0.040}
& \cellcolor{oursgreen}\textbf{0.217} \\
\hline

\multirow{5}{*}{3}
& MC Dropout     & \textbf{43.89} & \textbf{1.889} & \textbf{0.184} & \textbf{0.756} \\
& Deep Ensemble  & 43.03 & 2.212 & 0.276 & 0.808 \\
& TTA            & 12.08 & 4.118 & 0.429 & 1.152 \\
& Laplace Redux  & 12.39 & 3.043 & 0.222 & 0.993 \\
& \cellcolor{oursgreen}\textbf{Ours}
& \cellcolor{oursgreen}\underline{43.11}
& \cellcolor{oursgreen}\underline{2.165}
& \cellcolor{oursgreen}\underline{0.202}
& \cellcolor{oursgreen}\underline{0.767} \\
\hline

\multirow{5}{*}{5}
& MC Dropout     & \textbf{17.22} & \underline{3.548} & \underline{0.263} & \underline{1.012} \\
& Deep Ensemble  & 15.87 & 4.012 & 0.266 & 1.045 \\
& TTA            & 11.51 & 4.011 & 0.353 & 1.089 \\
& Laplace Redux  & 11.44 & 3.863 & 0.288 & \textbf{0.970} \\
& \cellcolor{oursgreen}\textbf{Ours}
& \cellcolor{oursgreen}\underline{16.98}
& \cellcolor{oursgreen}\textbf{2.980}
& \cellcolor{oursgreen}\textbf{0.162}
& \cellcolor{oursgreen}1.085 \\
\hline

\end{tabular}%
}
\caption{
CIFAR-10-C Gaussian noise corruption results at severities \(1\), \(3\), and \(5\).
}
\label{tab:cifar_c_gaussian}
\end{table}

\section{Corruption Robustness}
\label{sec:cifar_c_corruption}

We evaluate classification uncertainty under input corruption using CIFAR-10-C with Gaussian noise at severities \(1\), \(3\), and \(5\).
This setting tests whether uncertainty estimates remain meaningful as the input distribution becomes degraded.
Table~\ref{tab:cifar_c_gaussian} reports the full uncertainty and calibration results, while Figure~\ref{fig:cifar_c_gaussian_accuracy} visualizes the accuracy degradation.

Under mild corruption, our method preserves strong predictive accuracy and obtains the best Brier score, indicating that stochastic activation trajectories provide useful uncertainty without strongly affecting the main prediction.
At moderate and severe corruption, MC Dropout achieves the highest accuracy, suggesting that stronger feature removal can act as a noise-robust regularizer under heavy corruption.
However, this robustness comes from explicitly masking activations and therefore removing part of the intermediate representation.
In contrast, our method introduces micro-level stochasticity through activation switching, which keeps the information flow mostly intact while still producing diverse predictions for uncertainty estimation.
As a result, our method remains close to MC Dropout in accuracy while obtaining competitive uncertainty metrics.
At the highest corruption severity, our method achieves the best NLL \cite{guo2017calibration} and ECE \cite{naeini2015obtaining}, while Laplace Redux obtains the best Brier score but with substantially lower accuracy.

\FloatBarrier

\section{OOD Uncertainty Distributions}
\label{sec:ood_uncertainty_distributions}

To further inspect the OOD behavior of stochastic activation switching, we visualize the distribution of uncertainty scores for ID and OOD samples. 
Figure~\ref{fig:ood_histograms} shows that ID samples usually form compact low-uncertainty distributions, while OOD samples shift toward larger and more dispersed uncertainty values. 
This indicates that the model remains confident on familiar inputs and becomes less certain when the input distribution changes. 
The separation is especially clear for DRIVE$\rightarrow$CHASE-DB1, CIFAR-10$\rightarrow$SVHN, and CREMI$\rightarrow$ISBI, where OOD samples occupy a broader high-uncertainty range. 
For Massachusetts$\rightarrow$RoadTracer, the two distributions overlap more strongly because both datasets contain aerial road imagery with similar visual structure; nevertheless, the OOD distribution still shows a wider uncertainty spread. 
This supports the quantitative OOD results and shows that stochastic activation trajectories provide an interpretable signal for detecting distribution shift.

\begin{figure*}[h]
\centering
\includegraphics[width=0.46\textwidth]{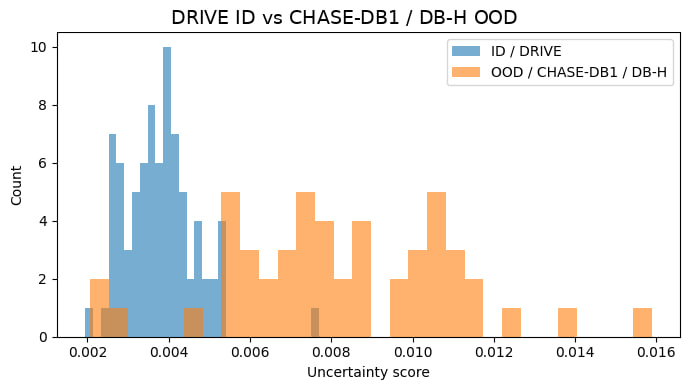}
\hfill
\includegraphics[width=0.46\textwidth]{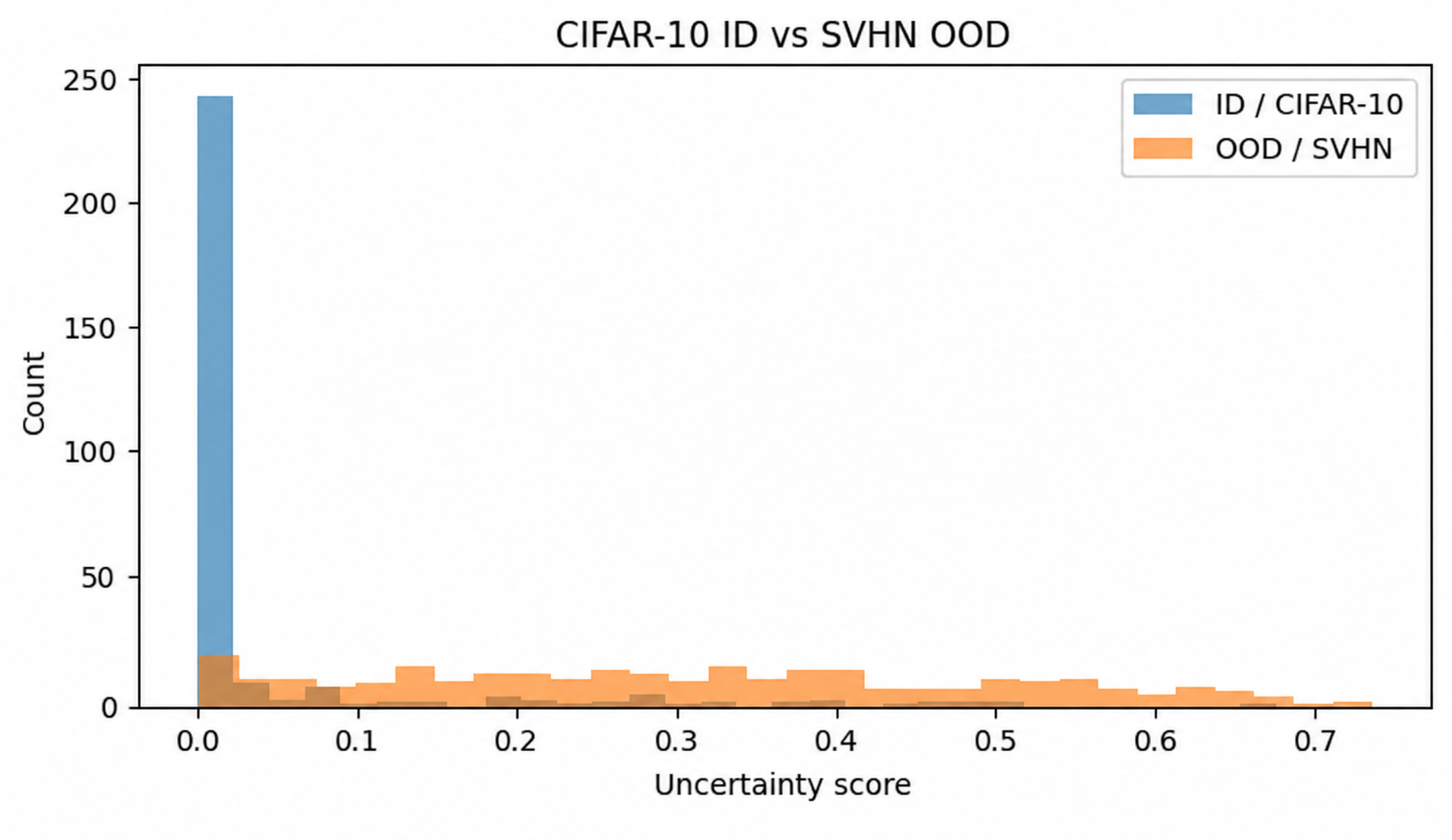}

\includegraphics[width=0.46\textwidth]{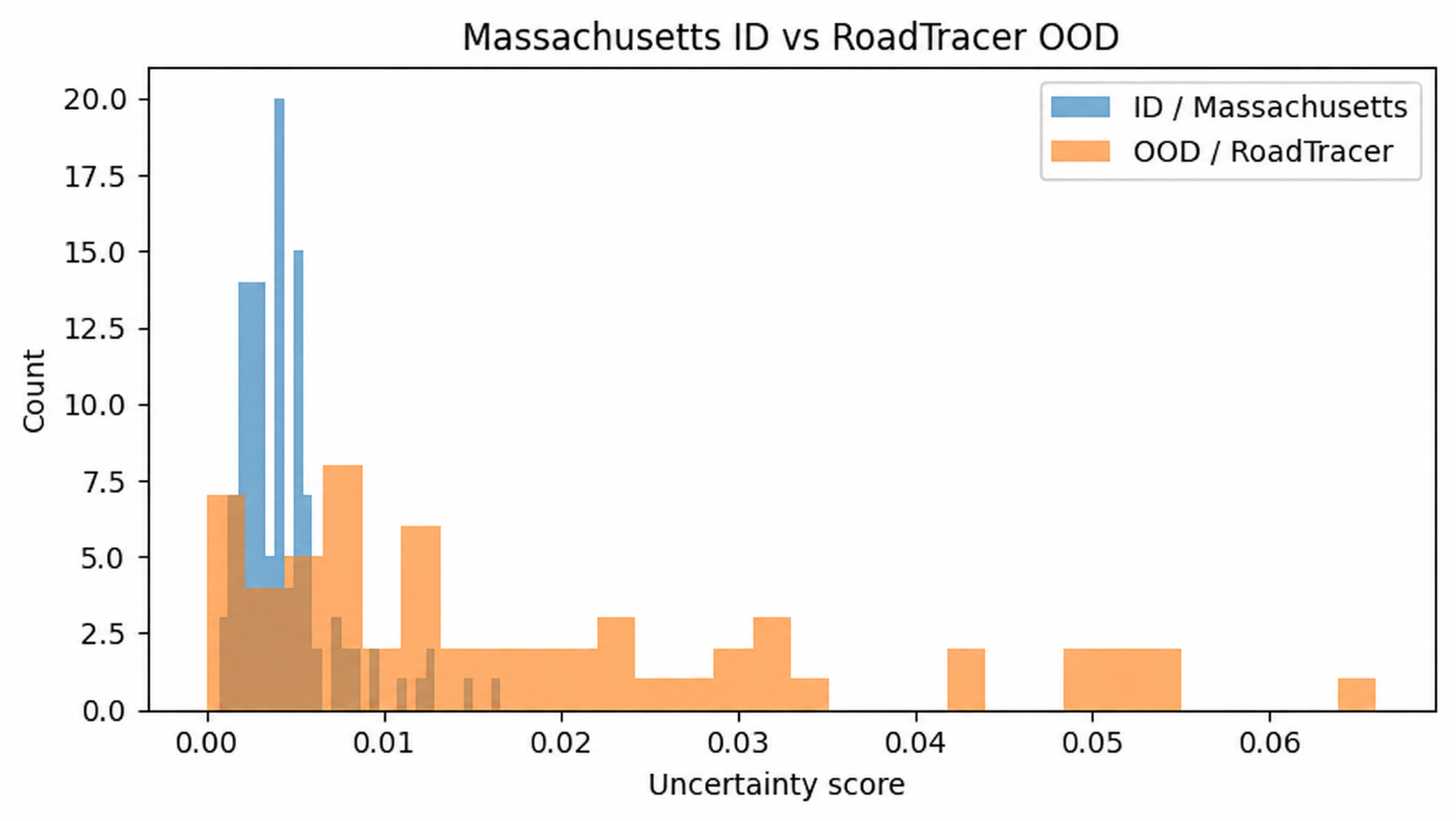}
\hfill
\includegraphics[width=0.46\textwidth]{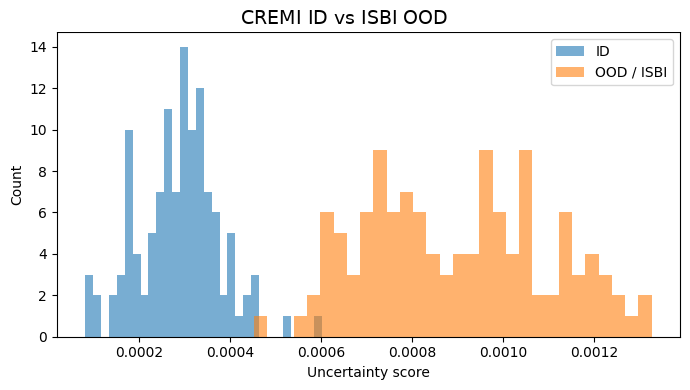}

\caption{
Uncertainty comparison of ID and OOD samples. 
ID scores are concentrated at lower uncertainty, while OOD scores are dispersed and shifted toward higher uncertainty.
}
\label{fig:ood_histograms}
\end{figure*}

\paragraph{}
Qualitative OOD segmentation examples are shown in Fig.~\ref{fig:isbi_ood_full_page}, where a model trained on CREMI is evaluated on ISBI 2012. 
Each row shows the original image, ground-truth mask, prediction, and variance-based uncertainty map. 
The uncertainty maps are concentrated around ambiguous membranes, boundary regions, and structures where the stochastic forward passes disagree. 
This provides spatial evidence that the model does not only assign a global OOD score, but also localizes where the cross-dataset shift creates uncertainty in the segmentation output.

\paragraph{}

\begin{figure*}[p]
\centering
\includegraphics[height=0.85\textheight,keepaspectratio]{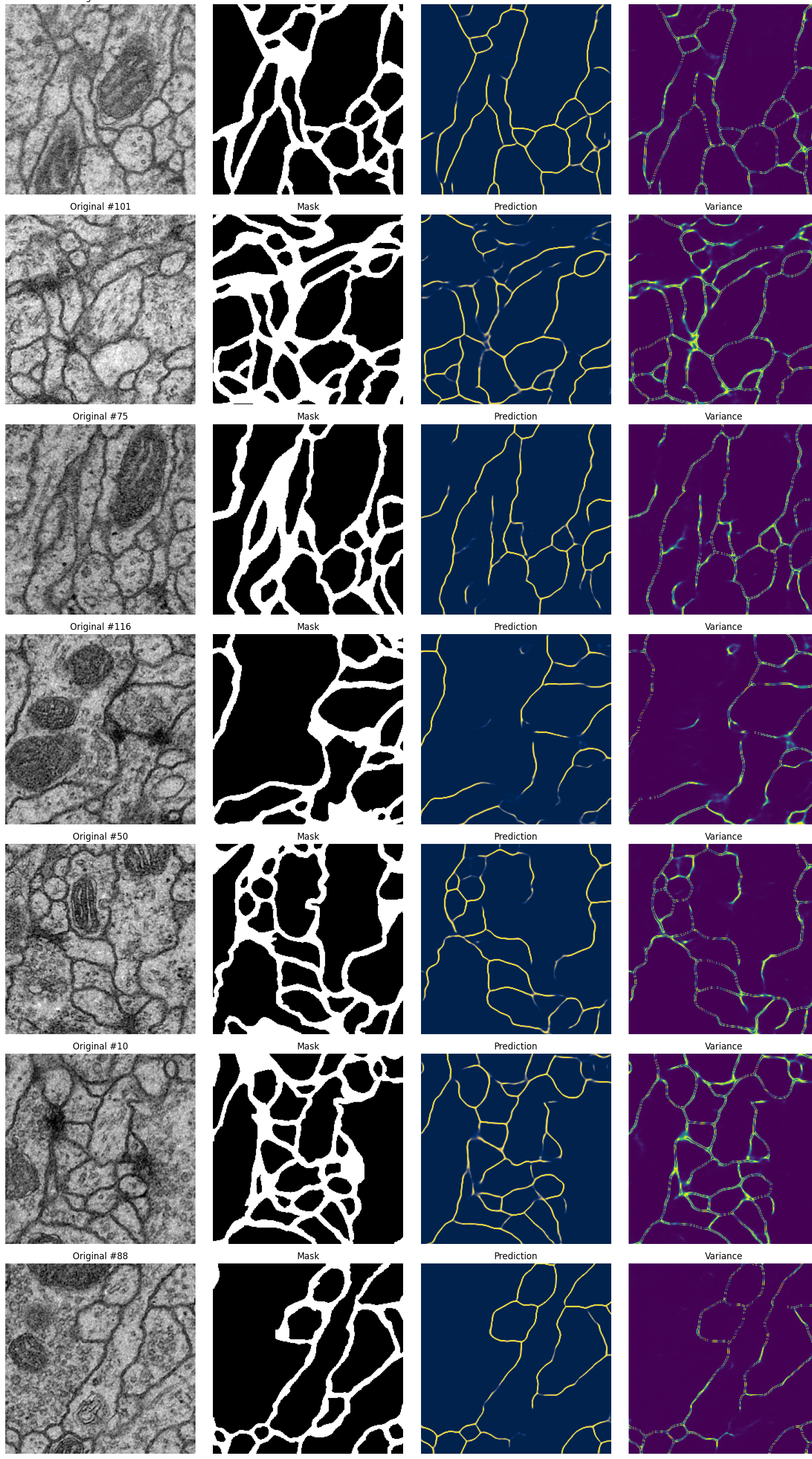}
\caption{
Qualitative OOD segmentation examples on ISBI 2012 using a model trained on CREMI. 
Each row shows the original image, ground-truth mask, prediction, and variance-based uncertainty map. 
}
\label{fig:isbi_ood_full_page}
\end{figure*}

%% file: bmvc_final.bbl
\begin{thebibliography}{67}
\providecommand{\natexlab}[1]{#1}
\providecommand{\url}[1]{\texttt{#1}}
\expandafter\ifx\csname urlstyle\endcsname\relax
  \providecommand{\doi}[1]{doi: #1}\else
  \providecommand{\doi}{doi: \begingroup \urlstyle{rm}\Url}\fi

\bibitem[Abdar et~al.(2021)Abdar, Pourpanah, Hussain, Rezazadegan, Liu,
  Ghavamzadeh, Fieguth, Cao, Khosravi, Acharya, et~al.]{abdar2021review}
Moloud Abdar, Farhad Pourpanah, Sadiq Hussain, Dana Rezazadegan, Li~Liu,
  Mohammad Ghavamzadeh, Paul Fieguth, Xiaochun Cao, Abbas Khosravi, U~Rajendra
  Acharya, et~al.
\newblock A review of uncertainty quantification in deep learning: Techniques,
  applications and challenges.
\newblock \emph{Information Fusion}, 76:\penalty0 243--297, 2021.

\bibitem[Bastani et~al.(2018)Bastani, He, Abbar, Alizadeh, Balakrishnan,
  Chawla, Madden, and DeWitt]{bastani2018roadtracerautomaticextractionroad}
Favyen Bastani, Songtao He, Sofiane Abbar, Mohammad Alizadeh, Hari
  Balakrishnan, Sanjay Chawla, Sam Madden, and David DeWitt.
\newblock {RoadTracer: Automatic Extraction of Road Networks from Aerial Images
  }.
\newblock In \emph{IEEE/CVF Conference on Computer Vision and Pattern
  Recognition}, 2018.

\bibitem[Bilic et~al.(2023)Bilic, Christ, Li, Vorontsov, Ben-Cohen, Kaissis,
  Szeskin, Jacobs, Mamani, Chartrand, et~al.]{bilic2023liver}
Patrick Bilic, Patrick Christ, Hongwei~Bran Li, Eugene Vorontsov, Avi
  Ben-Cohen, Georgios Kaissis, Adi Szeskin, Colin Jacobs, Gabriel
  Efrain~Humpire Mamani, Gabriel Chartrand, et~al.
\newblock The liver tumor segmentation benchmark (lits).
\newblock \emph{Medical Image Analysis}, 84:\penalty0 102680, 2023.

\bibitem[Blundell et~al.(2015)Blundell, Cornebise, Kavukcuoglu, and
  Wierstra]{blundell2015weight}
Charles Blundell, Julien Cornebise, Koray Kavukcuoglu, and Daan Wierstra.
\newblock Weight uncertainty in neural network.
\newblock In \emph{International Conference on Machine Learning}, 2015.

\bibitem[Brier(1950)]{glenn1950verification}
Glenn~W. Brier.
\newblock Verification of forecasts expressed in terms of probability.
\newblock \emph{Monthly Weather Review}, 78\penalty0 (1):\penalty0 1--3, 1950.

\bibitem[Cardona et~al.(2010)Cardona, Saalfeld, Preibisch, Schmid, Cheng,
  Pulokas, Tomancak, and Hartenstein]{cardona2010drosophila}
Albert Cardona, Stephan Saalfeld, Stephan Preibisch, Benjamin Schmid, Albert
  Cheng, John Pulokas, Pavel Tomancak, and Volker Hartenstein.
\newblock An integrated micro- and macroarchitectural analysis of the
  drosophila brain by computer-assisted serial section electron microscopy.
\newblock \emph{PLoS Biology}, 8\penalty0 (10):\penalty0 e1000502, 2010.

\bibitem[Choi et~al.(2026)Choi, Kim, Lee, jo, lee, and
  Ahn]{choi2026uncertainty}
Eunseo Choi, Ho-Yeon Kim, Jaewon Lee, Taeyong jo, Myungjun lee, and Heejin Ahn.
\newblock Uncertainty estimation via hyperspherical confidence mapping.
\newblock In \emph{International Conference on Learning Representations}, 2026.

\bibitem[Clevert et~al.(2016)Clevert, Unterthiner, and
  Hochreiter]{clevert2015fast}
Djork-Arn{\'e} Clevert, Thomas Unterthiner, and Sepp Hochreiter.
\newblock Fast and accurate deep network learning by exponential linear units
  (elus).
\newblock \emph{International Conference on Learning Representations}, 2016.

\bibitem[Conde et~al.(2022)Conde, Choi, Burchi, and Timofte]{conde2022swin2sr}
Marcos~V. Conde, Ui-Jin Choi, Maxime Burchi, and Radu Timofte.
\newblock Swin2sr: Swinv2 transformer for compressed image super-resolution and
  restoration.
\newblock In \emph{European Conference on Computer Vision}, 2022.

\bibitem[Daxberger et~al.(2021)Daxberger, Kristiadi, Immer, Eschenhagen, Bauer,
  and Hennig]{daxberger2021laplace}
Erik Daxberger, Agustinus Kristiadi, Alexander Immer, Runa Eschenhagen,
  Matthias Bauer, and Philipp Hennig.
\newblock Laplace redux-effortless bayesian deep learning.
\newblock \emph{Advances in Neural Information Processing Systems}, 2021.

\bibitem[Diederik(2014)]{diederik2014adam}
Kingma Diederik.
\newblock Adam: A method for stochastic optimization.
\newblock \emph{(No Title)}, 2014.

\bibitem[Duan et~al.(2024)Duan, Chapeau-Blondeau, and
  Abbott]{duan2024optimized}
Fabing Duan, Fran{\c{c}}ois Chapeau-Blondeau, and Derek Abbott.
\newblock Optimized injection of noise in activation functions to improve
  generalization of neural networks.
\newblock \emph{Chaos, Solitons \& Fractals}, 178:\penalty0 114363, 2024.

\bibitem[Durasov et~al.(2021)Durasov, Bagautdinov, Baque, and
  Fua]{durasov2021masksembles}
Nikita Durasov, Timur Bagautdinov, Pierre Baque, and Pascal Fua.
\newblock Masksembles for uncertainty estimation.
\newblock In \emph{IEEE/CVF Conference on Computer Vision and Pattern
  Recognition}, 2021.

\bibitem[Durasov et~al.(2024)Durasov, Oner, Donier, Le, and
  Fua]{durasov2024enabling}
Nikita Durasov, Doruk Oner, Jonathan Donier, Hieu Le, and Pascal Fua.
\newblock Enabling uncertainty estimation in iterative neural networks.
\newblock In \emph{International Conference on Machine Learning}, 2024.

\bibitem[Elfwing et~al.(2018)Elfwing, Uchibe, and Doya]{elfwing2018sigmoid}
Stefan Elfwing, Eiji Uchibe, and Kenji Doya.
\newblock Sigmoid-weighted linear units for neural network function
  approximation in reinforcement learning.
\newblock \emph{Neural Networks}, 107:\penalty0 3--11, 2018.

\bibitem[Ersoy and Erkmen(2021)]{Electrica_2021_486148}
Durmuş Ersoy and Burcu Erkmen.
\newblock A stochastic computing method for generating activation functions in
  multilayer feedforward neural networks.
\newblock \emph{Electrica}, 21\penalty0 (3):\penalty0 376--387, 2021.

\bibitem[Fraz et~al.(2012)Fraz, Remagnino, Hoppe, Uyyanonvara, Rudnicka, Owen,
  and Barman]{fraz2012ensemble}
Muhammad~Moazam Fraz, Paolo Remagnino, Andreas Hoppe, Bunyarit Uyyanonvara,
  Alicja~R Rudnicka, Christopher~G Owen, and Sarah~A Barman.
\newblock An ensemble classification-based approach applied to retinal blood
  vessel segmentation.
\newblock \emph{IEEE Transactions on Biomedical Engineering}, 59\penalty0
  (9):\penalty0 2538--2548, 2012.

\bibitem[Funke et~al.(2018)Funke, Tschopp, Grisaitis, Sheridan, Singh,
  Saalfeld, and Turaga]{funke2018large}
Jan Funke, Fabian Tschopp, William Grisaitis, Arlo Sheridan, Chandan Singh,
  Stephan Saalfeld, and Srinivas~C Turaga.
\newblock Large scale image segmentation with structured loss based deep
  learning for connectome reconstruction.
\newblock \emph{IEEE Transactions on Pattern Analysis and Machine
  Intelligence}, 41\penalty0 (7):\penalty0 1669--1680, 2018.

\bibitem[Gal and Ghahramani(2016)]{gal2016dropout}
Yarin Gal and Zoubin Ghahramani.
\newblock Dropout as a bayesian approximation: Representing model uncertainty
  in deep learning.
\newblock In \emph{International Conference on Machine Learning}, 2016.

\bibitem[Gal et~al.(2017)Gal, Hron, and Kendall]{gal2017concrete}
Yarin Gal, Jiri Hron, and Alex Kendall.
\newblock Concrete dropout.
\newblock \emph{Advances in Neural Information Processing Systems}, 2017.

\bibitem[Gast and Roth(2018)]{gast2018lightweight}
Jochen Gast and Stefan Roth.
\newblock Lightweight probabilistic deep networks.
\newblock In \emph{IEEE/CVF Conference on Computer Vision and Pattern
  Recognition}, 2018.

\bibitem[Goodfellow et~al.(2016)Goodfellow, Bengio, and
  Courville]{goodfellow2016deep}
Ian Goodfellow, Yoshua Bengio, and Aaron Courville.
\newblock \emph{Deep Learning}.
\newblock The MIT Press, 2016.

\bibitem[Gulcehre et~al.(2016)Gulcehre, Moczulski, Denil, and
  Bengio]{gulcehre2016noisy}
Caglar Gulcehre, Marcin Moczulski, Misha Denil, and Yoshua Bengio.
\newblock Noisy activation functions.
\newblock In \emph{International Conference on Machine Learning}, 2016.

\bibitem[Guo et~al.(2017)Guo, Pleiss, Sun, and Weinberger]{guo2017calibration}
Chuan Guo, Geoff Pleiss, Yu~Sun, and Kilian~Q. Weinberger.
\newblock On calibration of modern neural networks.
\newblock In \emph{International Conference on Machine Learning}, 2017.

\bibitem[Hammad(2024)]{hammad2024deep}
MM~Hammad.
\newblock Deep learning activation functions: fixed-shape, parametric,
  adaptive, stochastic, miscellaneous, non-standard, ensemble.
\newblock \emph{arXiv preprint arXiv:2407.11090}, 2024.

\bibitem[He et~al.(2016)He, Zhang, Ren, and Sun]{he2016deep}
Kaiming He, Xiangyu Zhang, Shaoqing Ren, and Jian Sun.
\newblock Deep residual learning for image recognition.
\newblock In \emph{Proceedings of the IEEE conference on computer vision and
  pattern recognition}, pages 770--778, 2016.

\bibitem[Hendrycks(2016)]{hendrycks2016gaussian}
D~Hendrycks.
\newblock Gaussian error linear units (gelus).
\newblock \emph{arXiv preprint arXiv:1606.08415}, 2016.

\bibitem[Hinton et~al.(2012)Hinton, Srivastava, Krizhevsky, Sutskever, and
  Salakhutdinov]{hinton2012improvingneuralnetworkspreventing}
Geoffrey~E. Hinton, Nitish Srivastava, Alex Krizhevsky, Ilya Sutskever, and
  Ruslan~R. Salakhutdinov.
\newblock Improving neural networks by preventing co-adaptation of feature
  detectors, 2012.

\bibitem[Holtzman et~al.(2020)Holtzman, Buys, Du, Forbes, and
  Choi]{holtzman2019curious}
Ari Holtzman, Jan Buys, Li~Du, Maxwell Forbes, and Yejin Choi.
\newblock The curious case of neural text degeneration.
\newblock In \emph{International Conference on Learning Representations}, 2020.

\bibitem[Huang et~al.(2017)Huang, Li, Pleiss, Liu, Hopcroft, and
  Weinberger]{huang2017snapshot}
Gao Huang, Yixuan Li, Geoff Pleiss, Zhuang Liu, John~E. Hopcroft, and Kilian~Q.
  Weinberger.
\newblock Snapshot ensembles: Train 1, get m for free.
\newblock In \emph{International Conference on Learning Representations}, 2017.

\bibitem[Kendall and Gal(2017)]{kendall2017uncertainties}
Alex Kendall and Yarin Gal.
\newblock What uncertainties do we need in bayesian deep learning for computer
  vision?
\newblock In \emph{Advances in Neural Information Processing Systems}, 2017.

\bibitem[Krizhevsky and Hinton(2009)]{krizhevsky2009learning}
Alex Krizhevsky and Geoffrey Hinton.
\newblock Learning multiple layers of features from tiny images.
\newblock Technical report, University of Toronto, 2009.

\bibitem[Krizhevsky et~al.(2012)Krizhevsky, Sutskever, and
  Hinton]{krizhevsky2012imagenet}
Alex Krizhevsky, Ilya Sutskever, and Geoffrey~E Hinton.
\newblock Imagenet classification with deep convolutional neural networks.
\newblock \emph{Advances in Neural Information Processing Systems}, 25, 2012.

\bibitem[Lakshminarayanan et~al.(2017)Lakshminarayanan, Pritzel, and
  Blundell]{lakshminarayanan2017simple}
Balaji Lakshminarayanan, Alexander Pritzel, and Charles Blundell.
\newblock Simple and scalable predictive uncertainty estimation using deep
  ensembles.
\newblock In \emph{Advances in Neural Information Processing Systems}, 2017.

\bibitem[LeCun et~al.(2002)LeCun, Bottou, Orr, and
  M{\"u}ller]{lecun2002efficient}
Yann LeCun, L{\'e}on Bottou, Genevieve~B Orr, and Klaus-Robert M{\"u}ller.
\newblock Efficient backprop.
\newblock In \emph{Neural Networks: Tricks of the Trade}, pages 9--50.
  Springer, 2002.

\bibitem[Liu et~al.(2020)Liu, Lin, Padhy, Tran, Bedrax~Weiss, and
  Lakshminarayanan]{liu2020simple}
Jeremiah Liu, Zi~Lin, Shreyas Padhy, Dustin Tran, Tania Bedrax~Weiss, and
  Balaji Lakshminarayanan.
\newblock Simple and principled uncertainty estimation with deterministic deep
  learning via distance awareness.
\newblock In \emph{Advances in Neural Information Processing Systems}, 2020.

\bibitem[Liu et~al.(2022)Liu, Hu, Lin, Yao, Xie, Wei, Ning, Cao, Zhang, Dong,
  et~al.]{liu2022swin}
Ze~Liu, Han Hu, Yutong Lin, Zhuliang Yao, Zhenda Xie, Yixuan Wei, Jia Ning, Yue
  Cao, Zheng Zhang, Li~Dong, et~al.
\newblock Swin transformer v2: Scaling up capacity and resolution.
\newblock In \emph{IEEE/CVF Conference on Computer Vision and Pattern
  Recognition}, 2022.

\bibitem[Lomeli et~al.(2025)Lomeli, Douze, Szilvasy, Cabannes, Copet,
  Sukhbaatar, Weston, Synnaeve, Mazar{\'e}, and
  J{\'e}gou]{lomeli2025stochastic}
Maria Lomeli, Matthijs Douze, Gergely Szilvasy, Loic Cabannes, Jade Copet,
  Sainbayar Sukhbaatar, Jason Weston, Gabriel Synnaeve, Pierre-Emmanuel
  Mazar{\'e}, and Herv{\'e} J{\'e}gou.
\newblock Stochastic activations.
\newblock \emph{arXiv preprint arXiv:2509.22358}, 2025.

\bibitem[Loquercio et~al.(2020)Loquercio, Segu, and
  Scaramuzza]{loquercio2020general}
Antonio Loquercio, Mattia Segu, and Davide Scaramuzza.
\newblock A general framework for uncertainty estimation in deep learning.
\newblock \emph{IEEE Robotics and Automation Letters}, 5\penalty0 (2):\penalty0
  3153--3160, 2020.

\bibitem[Loshchilov and Hutter(2017)]{loshchilov2017decoupled}
Ilya Loshchilov and Frank Hutter.
\newblock Decoupled weight decay regularization.
\newblock \emph{arXiv preprint arXiv:1711.05101}, 2017.

\bibitem[Malinin and Gales(2018)]{malinin2018predictive}
Andrey Malinin and Mark Gales.
\newblock Predictive uncertainty estimation via prior networks.
\newblock In \emph{Advances in Neural Information Processing Systems}, 2018.

\bibitem[Minaee et~al.(2021)Minaee, Boykov, Porikli, Plaza, Kehtarnavaz, and
  Terzopoulos]{minaee2021image}
Shervin Minaee, Yuri Boykov, Fatih Porikli, Antonio Plaza, Nasser Kehtarnavaz,
  and Demetri Terzopoulos.
\newblock Image segmentation using deep learning: A survey.
\newblock \emph{IEEE Transactions on Pattern Analysis and Machine
  Intelligence}, 44\penalty0 (7):\penalty0 3523--3542, 2021.

\bibitem[Misra(2020)]{misra2019mish}
Diganta Misra.
\newblock Mish: A self regularized non-monotonic activation function.
\newblock In \emph{British Machine Vision Conference}, 2020.

\bibitem[Mnih(2013)]{MnihThesis}
Volodymyr Mnih.
\newblock \emph{Machine Learning for Aerial Image Labeling}.
\newblock PhD thesis, University of Toronto, 2013.

\bibitem[Mobiny et~al.(2021)Mobiny, Yuan, Moulik, Garg, Wu, and
  Van~Nguyen]{mobiny2021dropconnect}
Aryan Mobiny, Pengyu Yuan, Supratik~K Moulik, Naveen Garg, Carol~C Wu, and Hien
  Van~Nguyen.
\newblock Dropconnect is effective in modeling uncertainty of bayesian deep
  networks.
\newblock \emph{Scientific Reports}, 11\penalty0 (1):\penalty0 5458, 2021.

\bibitem[Mukhoti et~al.(2021)Mukhoti, Dokania, Torr, and Gal]{mukhoti2020batch}
Jishnu Mukhoti, Puneet~K. Dokania, Philip Torr, and Yarin Gal.
\newblock On batch normalisation for approximate bayesian inference.
\newblock In \emph{Third Symposium on Advances in Approximate Bayesian
  Inference}, 2021.

\bibitem[Mukhoti et~al.(2023)Mukhoti, Kirsch, Van~Amersfoort, Torr, and
  Gal]{mukhoti2023deep}
Jishnu Mukhoti, Andreas Kirsch, Joost Van~Amersfoort, Philip~HS Torr, and Yarin
  Gal.
\newblock Deep deterministic uncertainty: A new simple baseline.
\newblock In \emph{IEEE/CVF Conference on Computer Vision and Pattern
  Recognition}, 2023.

\bibitem[Naeini et~al.(2015)Naeini, Cooper, and
  Hauskrecht]{naeini2015obtaining}
Mahdi~Pakdaman Naeini, Gregory Cooper, and Milos Hauskrecht.
\newblock Obtaining well calibrated probabilities using bayesian binning.
\newblock In \emph{AAAI conference on artificial intelligence}, 2015.

\bibitem[Netzer et~al.(2011)Netzer, Wang, Coates, Bissacco, Wu, Ng,
  et~al.]{netzer2011reading}
Yuval Netzer, Tao Wang, Adam Coates, Alessandro Bissacco, Baolin Wu, Andrew~Y
  Ng, et~al.
\newblock Reading digits in natural images with unsupervised feature learning.
\newblock In \emph{NIPS workshop on deep learning and unsupervised feature
  learning}, volume~5, page~7. Granada, 2011.

\bibitem[Ovadia et~al.(2019)Ovadia, Fertig, Ren, Nado, Sculley, Nowozin,
  Dillon, Lakshminarayanan, and Snoek]{ovadia2019can}
Yaniv Ovadia, Emily Fertig, Jie Ren, Zachary Nado, David Sculley, Sebastian
  Nowozin, Joshua Dillon, Balaji Lakshminarayanan, and Jasper Snoek.
\newblock Can you trust your model's uncertainty? evaluating predictive
  uncertainty under dataset shift.
\newblock \emph{Advances in Neural Information Processing Systems}, 2019.

\bibitem[Redmon et~al.(2016)Redmon, Divvala, Girshick, and
  Farhadi]{redmon2016lookonceunifiedrealtime}
Joseph Redmon, Santosh Divvala, Ross Girshick, and Ali Farhadi.
\newblock You only look once: Unified, real-time object detection.
\newblock In \emph{IEEE/CVF Conference on Computer Vision and Pattern
  Recognition}, 2016.

\bibitem[Ren et~al.(2015)Ren, He, Girshick, and
  Sun]{ren2016fasterrcnnrealtimeobject}
Shaoqing Ren, Kaiming He, Ross Girshick, and Jian Sun.
\newblock Faster r-cnn: Towards real-time object detection with region proposal
  networks.
\newblock In \emph{Advances in Neural Information Processing Systems}, 2015.

\bibitem[Ronneberger et~al.(2015)Ronneberger, Fischer, and
  Brox]{ronneberger2015unetconvolutionalnetworksbiomedical}
Olaf Ronneberger, Philipp Fischer, and Thomas Brox.
\newblock U-net: Convolutional networks for biomedical image segmentation.
\newblock In \emph{Medical Image Computing and Computer-Assisted Intervention},
  2015.

\bibitem[Staal et~al.(2004)Staal, Abr{\`a}moff, Niemeijer, Viergever, and
  Van~Ginneken]{staal2004ridge}
Joes Staal, Michael~D Abr{\`a}moff, Meindert Niemeijer, Max~A Viergever, and
  Bram Van~Ginneken.
\newblock Ridge-based vessel segmentation in color images of the retina.
\newblock \emph{IEEE Transactions on Medical Imaging}, 23\penalty0
  (4):\penalty0 501--509, 2004.

\bibitem[Stoyanova et~al.(2023)Stoyanova, Ghandi, and
  Tavakol]{stoyanova2023toward}
Yana Stoyanova, Soroush Ghandi, and Maryam Tavakol.
\newblock Toward robust uncertainty estimation with random activation
  functions.
\newblock In \emph{Proceedings of the AAAI Conference on Artificial
  Intelligence}, volume~37, pages 15152--15160, 2023.

\bibitem[Teerapittayanon et~al.(2016)Teerapittayanon, McDanel, and
  Kung]{teerapittayanon2016branchynet}
Surat Teerapittayanon, Bradley McDanel, and Hsiang-Tsung Kung.
\newblock Branchynet: Fast inference via early exiting from deep neural
  networks.
\newblock In \emph{International Conference on Pattern Recognition}, 2016.

\bibitem[Van~Amersfoort et~al.(2020)Van~Amersfoort, Smith, Teh, and
  Gal]{van2020uncertainty}
Joost Van~Amersfoort, Lewis Smith, Yee~Whye Teh, and Yarin Gal.
\newblock Uncertainty estimation using a single deep deterministic neural
  network.
\newblock In \emph{International Conference on Machine Learning}, 2020.

\bibitem[Vaswani et~al.(2017)Vaswani, Shazeer, Parmar, Uszkoreit, Jones, Gomez,
  Kaiser, and Polosukhin]{vaswani2017attention}
Ashish Vaswani, Noam Shazeer, Niki Parmar, Jakob Uszkoreit, Llion Jones,
  Aidan~N Gomez, {\L}ukasz Kaiser, and Illia Polosukhin.
\newblock Attention is all you need.
\newblock \emph{Advances in Neural Information Processing Systems}, 2017.

\bibitem[Vinyals et~al.(2016)Vinyals, Blundell, Lillicrap, Wierstra,
  et~al.]{vinyals2016matching}
Oriol Vinyals, Charles Blundell, Timothy Lillicrap, Daan Wierstra, et~al.
\newblock Matching networks for one shot learning.
\newblock \emph{Advances in Neural Information Processing Systems}, 2016.

\bibitem[Wang et~al.(2019)Wang, Li, Aertsen, Deprest, Ourselin, and
  Vercauteren]{wang2019aleatoric}
Guotai Wang, Wenqi Li, Michael Aertsen, Jan Deprest, S{\'e}bastien Ourselin,
  and Tom Vercauteren.
\newblock Aleatoric uncertainty estimation with test-time augmentation for
  medical image segmentation with convolutional neural networks.
\newblock \emph{Neurocomputing}, 338:\penalty0 34--45, 2019.

\bibitem[Wen et~al.(2020)Wen, Tran, and Ba]{wen2020batchensemble}
Yeming Wen, Dustin Tran, and Jimmy Ba.
\newblock Batchensemble: an alternative approach to efficient ensemble and
  lifelong learning.
\newblock In \emph{International Conference on Learning Representations}, 2020.

\bibitem[Wiedemann et~al.(1998)Wiedemann, Heipke, Mayer, and
  Jamet]{wiedemann1998empirical}
Christian Wiedemann, Christian Heipke, Helmut Mayer, and Olivier Jamet.
\newblock Empirical evaluation of automatically extracted road axes.
\newblock \emph{Empirical evaluation techniques in computer vision},
  12:\penalty0 172--187, 1998.

\bibitem[Wu et~al.(2021)Wu, Li, Wang, Meng, Qin, Chen, Zhang, Liu,
  et~al.]{wu2021r}
Lijun Wu, Juntao Li, Yue Wang, Qi~Meng, Tao Qin, Wei Chen, Min Zhang, Tie-Yan
  Liu, et~al.
\newblock R-drop: Regularized dropout for neural networks.
\newblock \emph{Advances in Neural Information Processing Systems}, 2021.

\bibitem[Xu et~al.(2015)Xu, Wang, Chen, and Li]{xu2015empirical}
Bing Xu, Naiyan Wang, Tianqi Chen, and Mu~Li.
\newblock Empirical evaluation of rectified activations in convolutional
  network.
\newblock \emph{arXiv preprint arXiv:1505.00853}, 2015.

\bibitem[Yang et~al.(2021)Yang, Shi, and Ni]{yang2021medmnist}
Jiancheng Yang, Rui Shi, and Bingbing Ni.
\newblock Medmnist classification decathlon: A lightweight automl benchmark for
  medical image analysis.
\newblock In \emph{IEEE 18th International Symposium on Biomedical Imaging},
  2021.

\bibitem[Yang et~al.(2023)Yang, Shi, Wei, Liu, Zhao, Ke, Pfister, and
  Ni]{yang2023medmnist}
Jiancheng Yang, Rui Shi, Donglai Wei, Zequan Liu, Lin Zhao, Bilian Ke,
  Hanspeter Pfister, and Bingbing Ni.
\newblock Medmnist v2-a large-scale lightweight benchmark for 2d and 3d
  biomedical image classification.
\newblock \emph{Scientific Data}, 10\penalty0 (1):\penalty0 41, 2023.

\bibitem[Zeevi et~al.(2025)Zeevi, Shwartz-Ziv, LeCun, Staib, and
  Onofrey]{zeevi2025rate}
Tal Zeevi, Ravid Shwartz-Ziv, Yann LeCun, Lawrence~H Staib, and John~A Onofrey.
\newblock Rate-in: Information-driven adaptive dropout rates for improved
  inference-time uncertainty estimation.
\newblock In \emph{IEEE/CVF Conference on Computer Vision and Pattern
  Recognition}, pages 20757--20766, 2025.

\end{thebibliography}
